\documentclass{article}

\usepackage{microtype}
\usepackage{graphicx}
\usepackage{subfigure}
\usepackage{booktabs} % for professional tables

\usepackage{hyperref}

\usepackage[accepted]{icml2023}

\usepackage{amsmath}
\usepackage{amssymb}
\usepackage{mathtools}
\usepackage{amsthm}

\usepackage[capitalize,noabbrev]{cleveref}

\theoremstyle{plain}
\newtheorem{theorem}{Theorem}[section]

\newtheorem{lemma}[theorem]{Lemma}

\theoremstyle{definition}
\newtheorem{definition}[theorem]{Definition}

\theoremstyle{remark}

\usepackage[textsize=tiny]{todonotes}

\usepackage{amsfonts, amsmath, amsthm, enumerate, gensymb, mathtools, tikz, graphicx, thmtools, algorithmic, algorithm, caption, wrapfig}
\usepackage{adjustbox}
\usepackage{booktabs}
\usepackage{afterpage}
\usetikzlibrary{positioning,chains,fit,shapes,calc}

\usepackage{thmtools} 
\usepackage{thm-restate}
\usepackage{algorithm}

\usepackage{amssymb}
\usepackage{amsmath}
\newif{\ifappendix}
\appendixtrue

\newcommand{\leaves}{\mathrm{leaves}}

\newcommand{\cost}{\mathrm{cost}}
\newcommand{\rot}{\mathrm{root}}
\newcommand{\clust}{\mathrm{cluster}}
\newcommand{\lft}{\mathrm{left}}
\newcommand{\rgt}{\mathrm{right}}

\newcommand{\treer}{\mathrm{RebalanceTree}}
\newcommand{\treerr}{\mathrm{RefineRebalanceTree}}
\newcommand{\treerm}{\mathrm{StochasticallyFairHC}}

\newcommand{\fhcp}{\textsf{FHCP}}

\newcommand{\search}{\mathrm{SubtreeSearch}}
\newcommand{\fairhc}{\mathrm{FairHC}}

\newcommand{\col}{\lambda}
\newcommand{\colr}{\ell}
\newcommand{\polylog}{\mathrm{polylog}}
\newcommand{\depth}{d}
\newcommand{\itvar}{y}

\newcommand{\optreer}{\mathrm{tree\_rebalance}}
\newcommand{\opdelins}{\mathrm{del\_ins}}
\newcommand{\opabs}{\mathrm{abstract}}
\newcommand{\opfold}{\mathrm{fold}}

\newcommand{\prflemoprebalance}{
\begin{proof}[Proof of Lemma~\ref{lem:oprebalance}]
Let $e = (x,y)$ be an edge that is separated by a tree rebalance operator $\optreer(u,v)$ for some internal nodes $u$ and $v$. Let's consider when we execute the rebalance. Let $V = \clust(v)$ be the set of vertices corresponding to $V$. Traverse down the tree from $v$ to $u$. Label the clusters we come across $A_1,A_2,\ldots,A_{k-1}$ and their corresponding un-traversed children $B_1,B_2,\ldots, B_{k-1}$. Let $A_k = \clust(u)$, and $B_k$ be the cluster for its only sibling.

When we rebalance, our first split will now divide $V$ into $A_k$ and $B:=\cup_{i\in[k]} B_i$. For $e$ to be separated by the rebalance of $u$ with respect to $v$, it must be that $x\in \clust(u)$ and $y\in B = \clust(a) \setminus \clust(u)$ (without loss of generality). This means that their lowest common ancestor was on the path between $u$ and $v$ (excluding $u$), which means the smallest $n_T(e)$ could be is $n_T(p)$ where $p$ is the parent of $u$. That means $\cost_T(e) = w(e)\cdot n_T(p)$.

In $T'$, their lowest common ancestor is $v$, thus $n_{T'}(e) = n_{T'}(v) = n_T(v)$, following from the observation that $v$'s cluster does not change. Thus, $\cost_{T'}(e) \leq w(e) \cdot n_T(v)$. Putting these together, we find $\cost_{T'}(e) \leq \frac{n_T(v)}{n_T(p)}\cost_T(e)$.% = \ell\cost_T(e)$.
\end{proof}
}

\newcommand{\prfopdelins}{
\begin{proof}[Proof of Lemma~\ref{lem:opdelins}]
Let $e = (x,y)$ be an edge that is separated by a subtree deletion and insertion operators $\opdelins(u,v)$ for some appropriate internal nodes $u$ and $v$. Let's consider when we execute the subtree deletion and insertion. For $x$ and $y$ to be separated, $x$ must be in $\clust(u)$ and $y$ must be in $\clust(u\land v)\setminus \clust(u)$ (without loss of generality). The first part is true because only the subtree $T[u]$ is moved, otherwise their least common ancestor would be unaffected. The second part is true because otherwise $y$ is either in $T[u]$ too, in which case their relative position remains the same in the subtree, or $y\notin T[u\land v]$, in which case still the move still does not affect their least common ancestor (which is higher in the tree than $u\land v$).

Now, since $x\in T[u]$ and $y\notin T[u]$, $x\land y$ must be an ancestor of $u$, thus $n_T(e) \geq n_T(u)$. This means that $\cost_T(e) \geq w(e)\cdot n_T(u)$. In $T'$, their least common ancestor must still remain below $u\land v$, since all the points in $T[u\land v]$ remain somewhere below $u\land v$. Also note no points are added to $T[u\land v]$ over the two operators. Thus $n_{T'}(e) \leq n_T'(u\land v) = n_T(u\land v)$. This means $\cost_{T'}(e) \leq w(e) \cdot n_T(u\land v)$, so $\cost_{T'}(e) \leq \frac{n_T(u\land v)}{n_T(u)} \cost_T(e)$. Thus, $\Delta = \frac{n_T(u\land v)}{n_T(u)}$.
\end{proof}
}

\newcommand{\prfabstract}{
\begin{proof}[Proof of Lemma~\ref{lem:opabstract}]
Let $e = (x,y)$ be an edge that is separated by a level abstraction operator $\opabs(h_1,h_2)$ for some depths $h_1$ and $h_2$ with $h_1< h_2$. Let's consider when we execute the abstraction. For $x$ and $y$ to be separated, $x\land y$ must be merged into its parent by the operator. That means it is between depth $h_1$ and $h_2$. Let $v$ be the vertex with the smallest $n_T(v)$ between depths $h_1$ and $h_2$. Then $n_T(x\land y) \geq n_T(v)$, and so $\cost_T(e) \geq w(e) \cdot n_T(v)$.

The ancestor it eventually gets contracted into must be of depth $h_1$, because we stop contracting after that point. Although its tree structure is altered below it, its cluster size remains the same since no vertices are moved away or to its subtree. Let $u$ be the vertex with the largest $n_T(u)$ between depths $h_1$ and $h_2$. Then we get $n_{T'}(x\land y) \geq n_T(u)$, and so $\cost_{T'}(e) \leq w(e) \cdot n_T(u)$.

This has shown us that $\cost_{T'}(e) \leq \frac{n_T(u)}{n_T(v)}\cost_{T}(e)$. Notice that $u$ and $v$ are precisely the internal nodes that maximize the ratio, so %$d = \frac{n_T(u)}{n_T(v)}$, and 
$\cost_{T'}(e) \leq \frac{n_T(u)}{n_T(v)}\cost_T(e)$.

\end{proof}
}

\newcommand{\prffold}{
\begin{proof}[Proof of Lemma~\ref{lem:opfold}]
Let $e = (x,y)$ be an edge that is separated by a tree folding operator $\opfold(T_1,\ldots,T_k)$ for subtrees $T_1,\ldots, T_k$ of $T$ satisfying the operator conditions. Let's consider when we execute the folding. For $x$ and $y$ to be separated, $x\land y$ must be in one of the subtrees, say $T_1$ without loss of generality. This means $\cost_T(e) \geq w(e)\cdot n_T(x\land y)$.

Now we consider the cost in $T'$. Clearly, $x\land y$ becomes the single vertex in $T_f$ corresponding to $\phi_1(x\land y)$. A leaf vertex in $T_2$ (without loss of generality) is only a descendant of $\phi_1(x\land y)$ if it has an ancestor $a$ such that $\phi_2(a) = \phi_1(x\land y)$. Therefore:

\[n_{T'}(x\land y) = n_{T'}(\phi_1(x\land y)) \leq \sum_{i\in[k]} n_T(\phi_i^{-1}(\phi_1(x\land y))\]

If $u = \max\{n_T(u): u\in T_i, i\in[k], \phi_i(u) = \phi_1(x\land y)\}$, then we further have:

\[n_{T'}(x\land y) \leq \sum_{i\in [k]}n_T(u) = kn_T(u)\]

This means that $\cost_{T'}(e) \leq w(e)\cdot kn_T(u)$. Putting these together gives $\cost_{T'}(e) \leq \frac{n_T(x\land y)}{n_T(v)}k\cost_T(e) \leq \frac{n_T(u)}{n_T(v)}k\cost_T(e)$ where $u$ and $v$ are the vertices merged together that maximize this ratio. %By definition, since $\frac{n_T(x\land y)}{n_T(v)} \leq \frac{n_T(u)}{n_T(v)}$ where $u$ and $v$ define the tree folding operation cost. Therefore, $\cost_{T'}(e) \leq \rho k\cost_T(e)$.

\end{proof}
}
\newcommand{\codetreer}{
\begin{algorithm}[H]
 \renewcommand{\algorithmicrequire}{\textbf{Input: }}
	\renewcommand{\algorithmicensure}{\textbf{Output: }}
    \caption{$\treer$}
    \label{alg:rebalance}
    \begin{algorithmic}[1]
    	\REQUIRE A hierarchy tree $T$ of size $n$, with smaller cluster always on the left.
        \ENSURE A $\frac16$-rebalanced $T$-tree.
        \STATE $r,v = \rot(T)$
        \STATE $A = \leaves(\lft_T(v))$
        \WHILE{$|A| \geq \frac 23 n$}{
            \STATE $v \gets \lft_T(v)$
            \STATE $A \gets \leaves(\lft_T(v))$
        }
        \ENDWHILE
        \STATE $T \gets T.\optreer(v,r)$
        \STATE Let $T'_l = \treer(T[\lft_T(r)])$
        \STATE Let $T'_r = \treer(T[\rgt_T(r)])$
        \STATE Return $T'$ with root $r$ with $\lft(r) = \rot(T_l')$ and $\rgt(r) = \rot(T_r')$
    \end{algorithmic}
\end{algorithm}
\vspace{-3mm}
}

\newcommand{\codesearch}{

\begin{algorithm}
 \renewcommand{\algorithmicrequire}{\textbf{Input: }}
	\renewcommand{\algorithmicensure}{\textbf{Output: }}
    \caption{$\search$}
    \label{alg:search}
    \begin{algorithmic}[1]
    	\REQUIRE A $\frac16$-relatively balanced hierarchy tree $T$ of size $n$, with smaller cluster always on the left and error parameter $s$.
        \ENSURE Modified $\frac16$-relatively balanced $T$ by a subtree deletion and insertion of a subtree of size between $s/3$ and $s$
        \STATE $v = \rot(T)$
        \WHILE{$|\leaves(\lft_T(v))| > s$}{
            \STATE $v \gets \rgt_T(v)$
        }
        \ENDWHILE
        \STATE $v\gets \lft_T(v)$
        \STATE 
        \STATE $u\gets \rot(T)$
        \WHILE{$|\leaves(\rgt_T(u))|\geq |\leaves(v)|$}{
            \STATE $u\gets \lft_T(u)$
        }
        \ENDWHILE
        \STATE $T' \gets T.\opdelins(u,v)$
        \STATE Return   $T$
    \end{algorithmic}
\end{algorithm}

}

\newcommand{\coderefine}{
\vspace{-1mm}
\begin{algorithm}[H]
 \renewcommand{\algorithmicrequire}{\textbf{Input: }}
	\renewcommand{\algorithmicensure}{\textbf{Output: }}
    \caption{$\treerr$}
    \label{alg:refinerebalance}
    \begin{algorithmic}[1]
    	\REQUIRE A $\frac16$-relatively balanced hierarchy tree $T$ of size $n$, with smaller cluster always on the left, and balance parameter $\epsilon \in (0,1/6)$.
        \ENSURE An $\epsilon$-relatively balanced tree.
        \IF{$\epsilon \leq 1/(2n)$}{
            \STATE Return  $T$
        }
        \ENDIF
        \STATE $v = \rot(T)$
        \STATE Let $T_{big} = T[\lft_T(v)]$, $T_{small} = T[\rgt_T(v)]$
        \WHILE{$|\leaves(T_{big})| \geq (1/2+\epsilon)n$}{
            \STATE $\delta \gets (|\leaves(T_{big})|-n/2)/n$
            \STATE Let $T_{big} = \search(T_{big}, \delta n)$
        }
        \ENDWHILE
        \STATE $T_{big} \gets \treerr(T_{big}, \epsilon)$
        \STATE $T_{small}\gets \treerr(T_{small}, \epsilon)$
        \STATE Return $T'$ with root $r$ with $\lft(r) = \rot(T_{big})$ and $\rgt(r) = \rot(T_{small})$
    \end{algorithmic}
\end{algorithm}
\vspace{-3mm}
}

\newcommand{\codefhc}{
\begin{algorithm}[H]
 \renewcommand{\algorithmicrequire}{\textbf{Input: }}
	\renewcommand{\algorithmicensure}{\textbf{Output: }}
    \caption{$\fairhc$}
    \label{alg:fhc}
    \begin{algorithmic}[1]
    	\REQUIRE An $\epsilon = 1/(c\log_2 n)$ relatively balanced hierarchy tree $T$ of size $n$ on red and blue points, and parameters $h=2^i$ and $k=2^j$ for some $0<j<i<\log_{1/2-\epsilon}(1/(2n\epsilon))$
        \ENSURE A fair tree. 
        \STATE Let $T \gets T.\opabs(0,i)$
        \IF{$T$ is height 1}{
        \STATE Return  $T$
        }
        \ENDIF
        \STATE Let $\mathcal{V}$ be the children of $\rot(T)$
        \FOR{each color $\ell\in[\lambda]$}{
        \STATE Order $\mathcal{V} = \{v_i\}_{i\in [h]}$ decreasing by $\frac{\ell(\leaves(v_i))}{|\leaves(v_i)|}$
        \STATE For all $i\in[k]$, $T \gets T.\opfold(\{T'[v_{i+(j-1)k}]: j\in[h/k]\})$
        }\ENDFOR
        \FOR{each child $v$ of $\rot(T)$}{
            \STATE Replace $T[v] \gets \fairhc(T[v])$
        }
        \ENDFOR
        \STATE Return $T'$
    \end{algorithmic}
\end{algorithm}
\vspace{-4mm}
}

\newcommand{\thmbalance}{
Given a $\gamma$-approximation for cost, we can construct a $\frac32\gamma$-approximation for cost which guarantees $\frac16$-relative balance. It only modifies the tree by applying tree rebalance operators of operation cost $3/2$, and every edge is only separated by at most one such operator.
}

\newcommand{\thmbalancetwo}{
Given a $\gamma$-approximation for cost, we can construct a $\frac{9\gamma}{2\epsilon}$-approximation for cost which guarantees $\epsilon$-relative balance for $0 < \epsilon\leq 1/6$. In addtition to Theorem~\ref{thm:balance}, it only modifies the tree by applying subtree deletion and insertion operators of operation cost $\frac3{\epsilon}$, and every edge is only separated by at most one such operator. 
}

\newcommand{\thmstoch}{
Given a $\gamma$-approximation for $\cost$ and any $\epsilon=1/(c\log_2n)$, $c,\lambda=O(1)$, and $\delta\in (0,1)$, in the stochastic fairness setting with $\frac{1}{1-\delta}\alpha_\ell \leq  p_\ell(v) \leq \frac1{1+\delta}\beta_\ell$ for all $v\in V$ and $\ell\in [\col]$, there is a $e^{4/(c(1-o(1))}\cdot \frac{3(1-\delta)\ln(cn)}{\delta^2\min_{\ell\in[\col]}\alpha_\ell}\cdot \frac{9\gamma}{2\epsilon}$ fair approximation that succeeds with high probability. On top of the operators of Theorem~\ref{thm:balance2}, it only modifies the tree by applying one level abstraction of operation cost at most $e^{4/(c(1-o(1))}\cdot \frac{3\ln(cn)}{a\delta^2}$.
}

\newcommand{\thmmain}{
Given a $\gamma$-approximation for $\cost$ over $\ell(V) = c_\ell n = O(n)$ vertices of each color $\ell\in[\lambda]$ with $h = n^\delta$ for any constants $c,\delta,k$, there is an algorithm that yields a hierarchy $T'$ that:
\begin{enumerate}
\item Is a $e^{\frac{4\log_2n}{c(1-o(1))\lambda\log_2h}+\frac 2c + \frac{4}{c(1-o(1))}} \cdot \frac{9c^2\gamma }{4}\cdot n^\delta\log_2^2n$-approximation for cost.
\item Is fair for any parameters for all $\ell\in[\lambda]$: $\beta_\ell \geq  c_\ell \left(\frac{ e^{4/c}}{kc_\ell} + e^{6/c}\right)^{ 1/\delta}$ and $\alpha_\ell \leq \frac{c_\ell}{e^{(6/c)\log_h(n)}}$.
\end{enumerate}
On top of the operators of  Theorem~\ref{thm:stoch}, it only modifies the tree by applying level abstraction of operation cost at most $e^{2/c}n^\delta$ and tree folding of operation cost $ke^{4/(c(1-o(1))}$ on $k$ subtrees, and each edge is separated in at most one level abstraction operator and in at most $\lambda/\delta$ tree fold operators.
}

\newcommand{\lemoprebalance}{
Given a tree $T$, let $T' = \optreer(u,v)$ for a node $u$ and an ancestor node $v$. The only edges separated by this are $e = (x,y)$ such that $x\in\clust(u)$ and $y \in \clust(a)\setminus\clust(u)$. The operation cost is bounded above by $\Delta = n_T(v)/n_T(p)$, where $p$ is the parent of $u$.
}

\newcommand{\lemopdelins}{
Given a tree $T$, let $T' = \opdelins(u,v)$ for two nodes $u$ and $v$, where $u$ is not an ancestor of $v$. The only edges separated by this are $e = (x,y)$ such that $x\in\clust(u)$ and $y \in \clust(u\land v)\setminus \clust(u)$. The operation cost is bounded above by $\Delta = n_T(u\land v)/n_T(u)$.
}

\icmltitlerunning{Generalized Reductions: Making any Hierarchical Clustering Fair and Balanced with Low Cost}

\begin{document}

\twocolumn[
\icmltitle{Generalized Reductions: Making any Hierarchical Clustering Fair and Balanced with Low Cost}

% It is OKAY to include author information, even for blind
% submissions: the style file will automatically remove it for you
% unless you've provided the [accepted] option to the icml2023
% package.

% List of affiliations: The first argument should be a (short)
% identifier you will use later to specify author affiliations
% Academic affiliations should list Department, University, City, Region, Country
% Industry affiliations should list Company, City, Region, Country

% You can specify symbols, otherwise they are numbered in order.
% Ideally, you should not use this facility. Affiliations will be numbered
% in order of appearance and this is the preferred way.
\icmlsetsymbol{equal}{*}

\begin{icmlauthorlist}
\icmlauthor{Marina Knittel}{umd}
\icmlauthor{Max Springer}{umd}
\icmlauthor{John Dickerson}{umd}
\icmlauthor{Mohammad Hajiaghayi}{umd}
\end{icmlauthorlist}

\icmlaffiliation{umd}{Department of Computer Science, University of Maryland, College Park, USA}

\icmlcorrespondingauthor{Marina Knittel}{mknittel@umd.edu}
\icmlcorrespondingauthor{Max Springer}{mss423@umd.edu}
\icmlcorrespondingauthor{John Dickerson}{john@cs.umd.edu}
\icmlcorrespondingauthor{Mohammad Hajiaghayi}{hajiagha@cs.umd.edu}

% You may provide any keywords that you
% find helpful for describing your paper; these are used to populate
% the "keywords" metadata in the PDF but will not be shown in the document
\icmlkeywords{Machine Learning, ICML}

\vskip 0.3in
]

% this must go after the closing bracket ] following \twocolumn[ ...

% This command actually creates the footnote in the first column
% listing the affiliations and the copyright notice.
% The command takes one argument, which is text to display at the start of the footnote.
% The \icmlEqualContribution command is standard text for equal contribution.
% Remove it (just {}) if you do not need this facility.

%\printAffiliationsAndNotice{}  % leave blank if no need to mention equal contribution
\printAffiliationsAndNotice{} % otherwise use the standard text.

\begin{abstract}
Clustering is a fundamental building block of modern statistical analysis pipelines. \emph{Fair} clustering has seen much attention from the machine learning community in recent years.
We are some of the first to study fairness in the context of hierarchical clustering, after the results of Ahmadian et al. from NeurIPS in 2020. We evaluate our results using Dasgupta's cost function, perhaps one of the most prevalent theoretical metrics for hierarchical clustering evaluation. Our work vastly improves the previous $O(n^{5/6}poly\log(n))$ fair approximation for cost to a near polylogarithmic $O(n^\delta poly\log(n))$ fair approximation for any constant $\delta\in(0,1)$. This result establishes a cost-fairness tradeoff and extends to broader fairness constraints than the previous work. We also show how to alter existing hierarchical clusterings to guarantee fairness and cluster balance across any level in the hierarchy.
\end{abstract}

\section{Introduction}

Fair machine learning, and namely clustering, has seen a recent surge as researchers recognize its practical importance. In spite of the clear and serious impact the lack of fairness in existing intelligent systems has on society~\cite{angwin2016machine,bogen2018help,ledford2019millions,sweeney2013discrimination}, and despite significant progress towards fair flat (not hierarchical) clustering \cite{ahmadian2020faircorr,backurs,bera2019fair,bera,brubach2020a,chakrabarti2021a,proportionalfairness,chierichetti2017fair,esmaeili2021fair,esmaeili2020probabilistic,kleindessner2,rosner}, fairness in hierarchical clustering has only received some recent attention~\cite{ahmadian2020fairhier,chhabra2020fair}. Thus, we are some of the first to study this problem. %This work proposes reduction-based techniques to construct fair hierarchical clusterings that exhibit nearly exponentially lower $\cost$ than previous work. Our algorithms also guarantee \textit{relative balance}, where clusters on the same hierarchical level are guaranteed to have similar size.

\begin{figure}
%\begin{wrapfigure}{r}{0.45\textwidth}
  \vspace{-3mm}
  \centering
  \includegraphics[width=6.3cm]{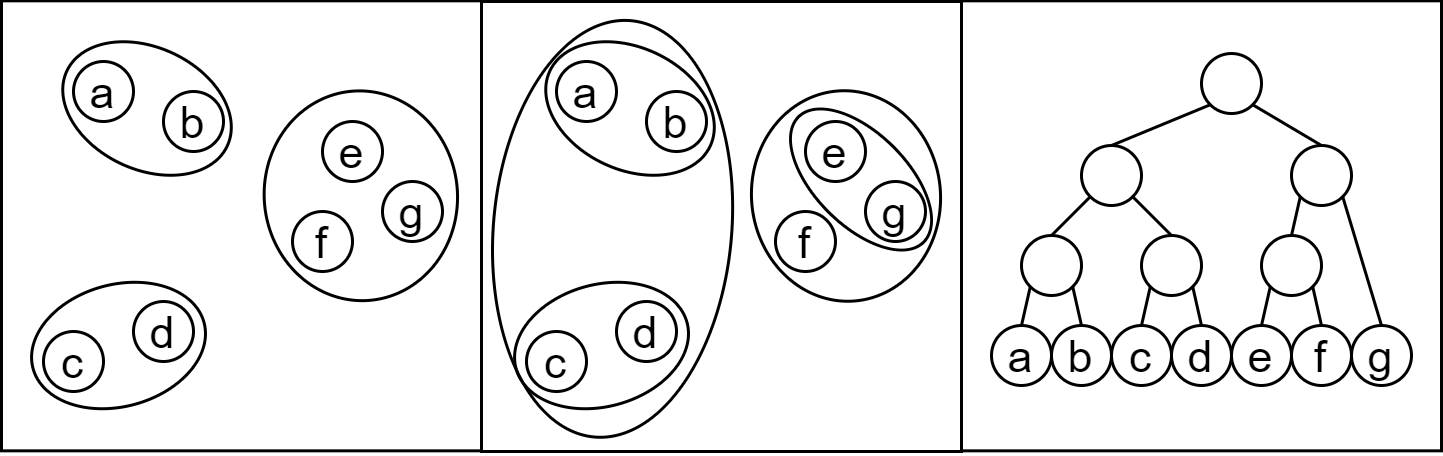}
   \caption{On the left is a 3-clustering, in the center is a hierarchical clustering, and on the right is its dendrogram. }
  \label{fig:simple}
  \vspace{-6mm}
%\end{wrapfigure}
\end{figure}

Hierarchical clustering (Figure~\ref{fig:simple}) is the well-known extension to clustering, where we create a hierarchy of subclusters contained within superclusters. This structure forms a tree (a dendrogram), where leaves represent the input data. An internal node $v$ corresponds to the cluster of all the leaves of the subtree rooted at $v$. The root is the cluster of all data.

Hierarchical clusterings more completely illustrate data relationships than flat clusterings. For instance, they are commonly used in phylogenetics to depict the entire evolutionary history of species, whereas a clustering would only depict species similarities. It also has a myriad of other uses in machine learning applications such as search~\cite{cai2004hierarchical,ferragina2005a,kou2012multiple}, social network analysis~\cite{leskovec2014mining,mann2008the}, and image recognition~\cite{arifin2006image,lin2018proximity,pan2016real}. On top of this, hierarchical clusterings can also be used to solve flat clustering when the number of clusters is not given. To do this, we extract clusterings at different resolutions in the hierarchy that all ``agree'' (if two points are together in a cluster, then they will also be together in any larger cluster) and select the one that best fits the application.

Hierarchical clusterings can be evaluated using a number of metrics. Perhaps most notably, \citet{dasgupta} introduced $\cost$ (Definition~\ref{def:cost}), an intuitive and explainable metric which exhibits numerous desirable properties and has become quite popular and well-respected~\cite{charikar17,charikar192,chatziafratis2018hierarchical,beyondworstcase,roy}. Unfortunately, it is difficult to approximate, where the best existing solutions require semi-definite programs~\cite{dasgupta,charikar17}, and it is not efficiently $O(1)$-approximable by the Small-Set Expansion Hypothesis~\cite{charikar17}. The $\mathrm{revenue}$~\cite{moseleywang} and $\mathrm{value}$~\cite{cohenaddad} metrics, both derived from $\cost$, exhibit $O(1)$-approximability, but are not as explainable or appreciated.

Only two papers have explored fair hierarchical clustering~\cite{ahmadian2020fairhier,chhabra2020fair}. Both extend fairness constraints from fair clustering literature that trace back to \citet{chierichetti2017fair}'s \textit{disparate impact}. Consider a graph $G = (V,E,w)$, where each point has a color, which represents a protected class (e.g., gender, race, etc.). On two colors, red and blue, they consider a clustering fair if the ratio between red and blue points in each cluster is equal to that in the input data. This ensures that the impact of a cluster on a protected class is proportionate to the class size. The constraint has been further  generalized~\cite{ahmadian,bercea2018cost}: given a dataset with $\col$ colors and constraint vectors $\vec{\alpha},\vec{\beta}\in (0,1)^\col$, a clustering is fair if for all $\ell\in[\col]$ and every cluster $C$, $\alpha_\ell|C| \leq \ell(C)\leq \beta_\ell|C|$, were $\ell(C)$ is the number of points in $C$ of color $\ell$. Naturally, then, a hierarchical clustering is fair if every non-singleton cluster in the hierarchy satisfies this constraint (with nuances, see Section~\ref{sec:prelimfair}), as in \citet{ahmadian2020fairhier}.

This work explores broad guarantees, namely $\cost$ approximations, for fair hierarchical clustering. The only previous algorithm is quite complicated and only yields a $O(n^{5/6}\log^{5/4}n)$ fair approximation for $\cost$ (where an $O(n)$-approximation is trivial)~\cite{ahmadian2020fairhier}, and it assumes two, equally represented colors. This reflects the inherent difficulty of finding  solutions that are low-$\cost$ as opposed to high-$\mathrm{revenue}$ or high-$\mathrm{value}$, both of which exhibit fair $O(1)$-approximations~\cite{ahmadian2020fairhier}). Our algorithms improve previous work in quite a few ways: 1) we achieve a near-exponential improvement in approximation factor, 2) our algorithm works on $O(1)$ instead of only 2 colors, 3) our work handles different representational proportions across colors in the initial dataset, 4) we simultaneously guarantee fairness and relative cluster balance, and 5) our methods, which modify a given (unfair) hierarchy, have measurable, explainable, and limited impacts on the structure of the input hierarchy.

%\begin{enumerate}
%\item We achieve a near-exponential improvement in approximation factor ($poly(n)$ to near-$poly\log(n)$).
%\item Our algorithm works on $O(1)$ instead of only 2 colors.
%\item Our work handles different representational proportions across colors in the initial dataset.
%\item We simultaneously guarantee fairness and relative cluster balance.
%\item Our methods, which modify a given (unfair) hierarchy, have measurable, explainable, and limited impacts on the structure of the input hierarchy.
%\end{enumerate}

\begin{table*}[]
\begin{center}
\begin{tabular}{|l|l|l|l|l|l|l|}
\hline
\textbf{} & Qualities Achieved                         & Approximation               & Fairness           & Colors & Color ratios & Explainable? \\ \hline
Previous  & Deterministic fairness                     & $O(n^{5/6}poly\log n)$      & Perfect            & 2      & 50/50 only   & No           \\ \hline
This work & Deterministic fairness & $O(n^\delta poly\log n)$    & Approximate       & $O(1)$ & $O(1)$       & Yes          \\ \cline{2-7} 
          & Stochastic fairness  & $O(\log^{3/2} n)$           & Approximate & $O(1)$ & $O(1)$       & Yes          \\ \cline{2-7} 
          & $\epsilon$-relative balance                         & $O(\sqrt{\log n}/\epsilon)$ &   N/A                 &   N/A     &     N/A         & Yes          \\ \cline{2-7} 
          & $1/6$-relative balance                             & $O(\sqrt{\log n})$          &   N/A                 &   N/A     &     N/A         & Yes          \\ \hline
\end{tabular}
\caption{Our versus previous work. Note $\delta\in (0,1/6)$ is parameterizable, trading approximation factor for fairness. Our algorithms are explainable in that the alterations made to the hierarchy are clear and well-defined.}
\label{table:contr}
\end{center}
\vspace{-6mm}
\end{table*}

\subsection{Our Contributions}
This work proposes new algorithms for fair and balanced hierarchical clustering. A summary of our work can be found in Table~\ref{table:contr}.

We introduce four simple hierarchy tree operators which have clear, measurable impacts. We show how to compose them  together on a (potentially unbalanced and unfair) hierarchy to yield a fair and/or balanced hierarchy with similar structure. This process clarifies the functionality of our algorithms and  illustrates the modifications imposed on the hierarchy. Each of our four proposed algorithms starts with a given $\gamma$-approximate (unfair) hierarchical clustering algorithm (i.e., \citet{dasgupta}'s $O(\sqrt{\log n})$-approximation) and then builds on top of each other, imposing a new operator to achieve a more advanced result. Additionally, each algorithm stands alone as a unique contribution.

Our first algorithm produces a $1/6$-relatively balanced hierarchy that $\frac32\gamma$-approximates $\cost$ (see Theorem~\ref{thm:balance}).\footnote{Repeated sparsest cuts achieves this with similar cost. Our algorithms, though, can be used explainably on top of existing unfair algorithms and may perform better as unfair research progresses.} Here, $\epsilon$-relative balance means that at each split in the hierarchy, a cluster splits in half within a proportional error of up to $1+\epsilon$ (see Definition~\ref{def:relbal}). Starting at the root, the algorithm recursively applies our \textit{tree rebalance} (see Definition~\ref{def:treer}) operator. This restructures the tree by moving some subtree up to become a child of the root. It preserves much of the hierarchy's structure while achieving relative balance.

\begin{figure}
%\begin{wrapfigure}{l}{0.53\textwidth}
  \vspace{-0mm}
  \centering
  \includegraphics[width=7.5cm]{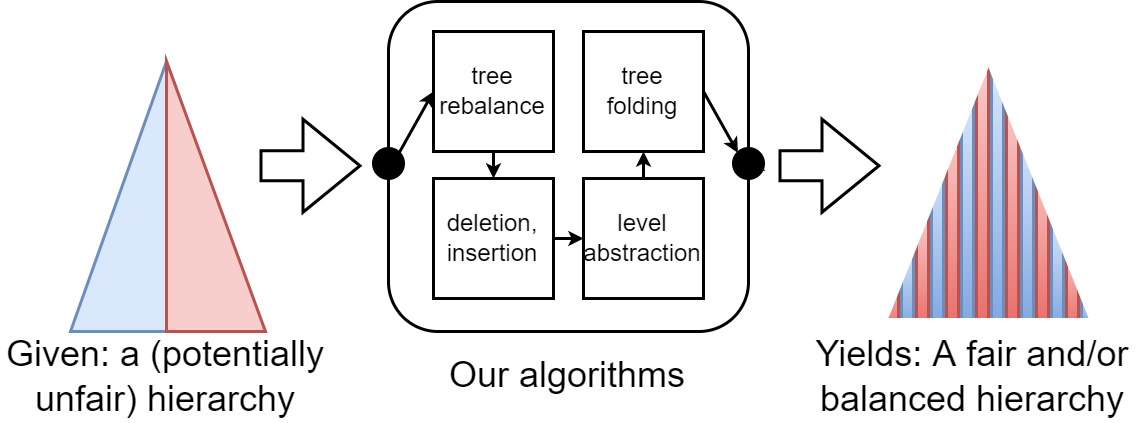}
   \caption{Our algorithms take a potentially unfair hierarchical clustering, apply our tree operators, and yield fair and/or balanced hierarchies.}
  \label{fig:reduction}
  \vspace{-6mm}
%\end{wrapfigure}
\end{figure}

Our next result refines this to achieve $\epsilon$-relative balance for any $\epsilon \in (0,1/6)$ that $\frac{9}{2\epsilon}\gamma$-approximates $\cost$ (see Theorem~\ref{thm:balance2}). This can get arbitrarily close to creating a perfectly balanced hierarchy. To achieve this, we simply run our first algorithm and then apply a limited number of \textit{subtree deletion and insertion operators} (see Definition~\ref{def:delins}). This operator selects a subtree, removes it, and reinserts elsewhere. It again preserves much of $T$'s structure.

Third, we propose an algorithm for stochastically fair hierarchical clustering (see Definition~\ref{def:stochexp}). Under certain stochastic parameterizations and arbitrarily many colors, the algorithm achieves stochastic fairness and $O(\gamma\log n)$-approximates $\cost$ (see Theorem~\ref{thm:stoch}). This is quite impressive, as the best previous fair approximation (albeit, for deterministic colors) was $poly(n)$~\cite{ahmadian2020fairhier}. To achieve this novel result, we first find an $O(1/\log n)$-relatively balanced hierarchy and then apply our level abstraction operator once to the bottom layers of the hierarchy. This operator removes selected layers, setting much lower descendants of a vertex as direct children. While this removes details in the hierarchy, the remaining structure still agrees with the original tree. This simple addition guarantees fairness under stochastic color assignment.

Our main result finds an approximately fair hierarchical clustering that $O(n^\delta poly\log n)$-approximates cost (see Theorem~\ref{thm:main}), where $\delta\in (0,1)$ is a given constant. This is a near-exponential improvement over previous work which only achieves a $O(n^{5/6}poly\log(n))$ approximation on two equally represented colors. On top of that, our algorithm works on many colors, with many different color ratios, and achieves a simultaneously balanced hierarchy in an explainable manner. The algorithm, $\fairhc$ (Algorithm~\ref{alg:fhc}), builds on top of our stochastic algorithm (parameterized slightly differently, see Section~\ref{sec:fair}) before applying a new operator: tree folding. Tree folding maps isomorphic trees on top of each other. In hierarchical clustering, this means taking two subtrees, mapping clusters in one tree to the other, and then merging clusters according to the mapping. Matching up clusters with different proportions of colors helps balance out the color ratios across clusters, which gives us our fairness result.

%\begin{corollary}\label{cor:main}
%Given a $\gamma$-approximation for $\cost$, on a $1/(c\log n)$-relatively balanced hierarchy tree $T$ 
%over $\ell(V) = c_\ell n = O(n)$ vertices of each color $\ell\in\lambda$ with $h = n^\delta$ for any constants $c,\delta$, there is an algorithm that yields a $O(n^\delta\gamma\log^2n)$ fair approximation for $\cost$ when $\beta_\ell \geq x_\ell c_\ell$ and $\alpha_\ell \leq (1-x_\ell)c_\ell$ for each $\ell\in[\lambda]$ where $x_R,x_B=O(1)$ and can be parameterized to be arbitrarily close to 1. This algorithm runs in $O(n^2\log n)$ time.
%\end{corollary}

\section{Preliminaries}\label{sec:prelim}

Our input is a complete weighted graph $G = (V, E, w)$ where $w:E\to \mathbb{R}^+$ is a similarity measure. A hierarchical clustering can be defined as a hierarchy tree $T$, where its leaves are $\leaves(T) = V$, and internal nodes represent the merging of vertices into clusters and clusters into superclusters.

\subsection{Optimization Problem}

We use \citet{dasgupta}'s $\cost$ function as an optimization metric. For simplicity, we let $n_T(e)$ denote the size of smallest cluster in $T$ containing both endpoints of $e$. In other words, for $e = (u,v)$, $n_T(e) = |\leaves(T[u\land v])|$, where $u\land v$ is the lowest common ancestor of $u$ and $v$ in $T$ and $T[u]$ for any vertex $u$ is the subtree rooted at $u$. We additionally denote $n_T(u) = |\leaves(T[u])|$ for internal node $u$. Also we let $\rot(T)$ be the root of $T$, and $\lft_T(u)$ and $\rgt_T(u)$ access left and right children respectively. We can evaluate the $\cost$ contribution of an edge to a hierarchy.

\begin{definition}
The \textbf{cost} of $e\in E$ in a graph $G=(V,E,w)$ in a hierarchy $T$ is $\cost_{T}(e) = w(e) \cdot n_{T}(e)$.
\end{definition}

We strive to minimize the sum of $\cost$s across all edges.

\begin{definition}[\citet{dasgupta}]\label{def:cost}
The \textbf{cost} of a hierarchy $T$ on graph $G = (V, E, w)$ is:
\[\cost(T) = \sum_{e\in E} \cost_T(e)\]
\end{definition}

\citet{dasgupta} showed that we can assume that all unfair trees optimizing for $\cost$ are binary. Note that we must consider non-binary trees when we incorporate fairness as it may not allow binary splits at its lowest levels.

\subsection{Fairness and Stochastic Fairness}\label{sec:prelimfair}

We consider the fairness constraints based off those introduced by \citet{chierichetti2017fair} and extended by \citet{bercea2018cost}. On a graph $G$ with colored vertices, let $\colr(C)$ count the number of $\colr$-colored points in cluster $C$.

\begin{definition}\label{def:fair}
Consider a graph $G = (V, E, w)$ with vertices colored one of $\col$ colors, and two vectors of parameters $\alpha,\beta \in (0,1)^\col$ with $\alpha_\colr \leq \beta_\colr$ for all $\colr\in[\col]$. A hierarchy $T$ on $G$ is \textbf{fair} if for any non-singleton cluster $C$ in $T$ and for every $\colr\in[\col]$, $\alpha_\colr|C| \leq \colr(C) \leq \beta_\colr |C|$. Additionally, any cluster with a leaf child has only leaf children.
\end{definition}

Notice that the final constraint regarding leaf-children simply enforces that a hierarchy must have some ``baseline'' fair clustering (e.g., a fairlet decomposition~\cite{chierichetti2017fair}). Consider a tree that is just a stick with individual leaf children at each depth. While internal nodes may represent fair clusters,  you cannot extract any nontrivial fair flat clustering from this, since it must contain a singleton, which is unfair. We view such a hierarchy This is clearly undesirable, and our additional constraint prevents this issue.

%Ahmadian et al.~\cite{ahmadian2020fairhier} introduced the notion of a \textit{union-closed} constraint, fairness or otherwise, which guarantees that the union of two clusters that satisfy the constraint each must satisfy it as well. It is not hard to see that our fairness constraint is union-closed. This means we only must show that clusters at height 1 in our hierarchy are fair. This ensures the rest of the hierarchy is fair.

In the stochastic problem, points are assigned colors at random. We must ensure that with high probability (i.e., at least $1-1/\polylog(n)$) all clusters are fair.

\begin{definition}\label{def:stochexp}
Consider the same context as Definition~\ref{def:fair} with an additional function $p_\colr:v\to (0,1)$ denoting the probability $v$ has color $\colr$ such that $\sum_{\colr=1}^\col p_\colr(v) = 1$ and each vertex has exactly one color. An algorithm is \textbf{stochastically fair} if, with high probability, it outputs a fair hierarchy.
\end{definition}

%This will be called the \textsf{StochasticallyFairHierarchicalClustering}, or \sfhc, problem. %We can also consider bounding it with high probability.

\iffalse
\begin{definition}
Consider the same context as Definition~\ref{def:stochexp}. A hierarchy $T$ on $G$ is \textbf{stochastically fair with high probability} if for any non-singleton cluster $C$ in $T$ and for every $i\in[c]$:
\[\alpha_i|C| \leq |P_i\cap C| \leq \beta_i |C|\]
with high probability.
\end{definition}

This will be called the \textsf{FairHierarchicalClusteringProbabilistic}, or \fhcp, problem.
\fi

\section{Tree Properties and Operators}

This work is interested in both fair and balanced hierarchies. Balanced trees have numerous practical uses, and in this paper, we show how to use them to guarantee fairness too.

\begin{definition}\label{def:relbal}
A hierarchy $T$ is \textbf{$\epsilon$-relatively balanced} if for every pair of clusters $C$ and $C'$ that share a parent cluster $C_p$ with $|C_p| \geq 1/(2\epsilon)$ in $T$, $(1/2-\epsilon) |C_p| \leq |C|,|C'|\leq (1/2+\epsilon)|C_p|$. 
\end{definition}

Notice that we only care about splitting clusters $C_p$ with size satisfying $|C_p| \geq 1/(2\epsilon)$. This is because, on smaller clusters, it may be impossible to divide them with relative balance. For instance, if $|C_p| = 3$, we know we can only split it into a 1-sized and 2-sized cluster, yielding a minimum relative balance of $1/6$. For smaller $\epsilon$, we require larger cluster sizes to make this possible.

We will often discuss the ``separation'' of edges in our proposed operators. It refers to occasions when a point is added to the first cluster that contains both endpoints. We do not care if points are removed.

\begin{definition}\label{def:sep}
An edge $e=(u,v)$ is (or its endpoints are) \textbf{separated} by an operator which changes hierarchy $T$ to $T'$ if $\clust_T(u\land v) \not\supseteq \clust_{T'}(u\land v)$.
\end{definition}

Almost definitionally, if an edge is not separated by an operator, then the cluster size at its lowest common ancestor does not increase. Thus, its cost contribution does not increase.

%\begin{observation}\label{obs:notsep}
%Consider a tree operator that transforms a tree $T$ into $T'$. Then for any edge $e$ in $T$, if $e$ is not separated by the operator, $\cost_{T'}(e) \leq \cost_T(e)$.
%\end{observation}

\subsection{Tree Operators}

Our work uses a number of different tree operations to modify and combine trees (Figure~\ref{fig:ops}). These illustrate exactly how our algorithms alter the input. We show how many operators of each type each of our algorithms use and to what extent they affect the hierarchy through a metric we propose here. Notably, for each proposed algorithm on an input $T$, it transforms $T$ into output $T'$ by \textit{only applying our four tree operators: tree rebalance, subtree deletion and insertion, level abstraction, and subtree folding.}

Each operator has an associated operation cost, which measures the proportional increase in $\cost$ of each edge separated by the operation. We present lemmas that bound the operation cost of each operator in the Appendix.

\begin{definition}
Assume we apply some tree operation to transform $T$ into $T'$. The \textbf{operation cost} is an upper bound $\Delta$ such that for any edge $e$ that is separated by the operation, then $\cost_{T'}(e) \leq \Delta \cost_{T}(e)$.
\end{definition}

The first operation is a tree rebalance, which rotates in a descendant of the root to instead be a direct child. This defines our first result in Theorem~\ref{thm:balance}, as clever use of the tree rebalance operator allows us to find a relatively balanced tree. This is illustrated in the top left panel of Figure~\ref{fig:ops}.

\begin{figure}
%\begin{wrapfigure}{r}{0.47\textwidth}
  \vspace{-0mm}
  \centering
  \includegraphics[width=6.3cm]{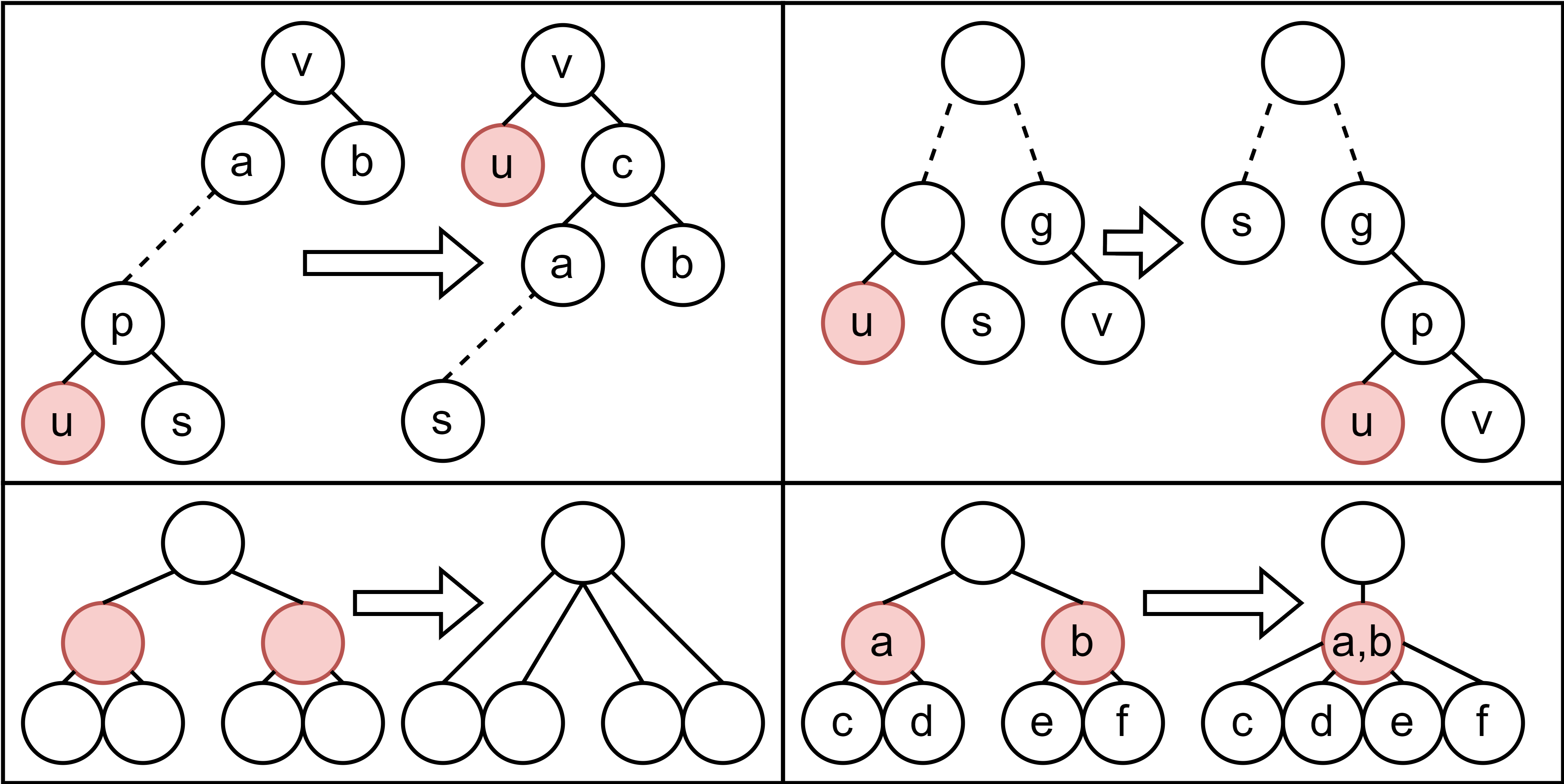}
   \caption{We depict our tree operators: tree rebalance (top left), subtree deletion and insertion (top right), level abstraction (bottom left), and tree folding (bottom right).}
  \label{fig:ops}
  \vspace{-4mm}
\end{figure}

\begin{definition}\label{def:treer}
Consider a binary tree $T$ with internal nodes $v$, $v$'s descendant $u$, and $v$'s children $a$ and $b$. A \textbf{tree rebalance} of $u$ at $v$ ($\optreer(u,v)$) puts a new node $c$ in between $v$ and sibling nodes $a$ and $b$. It then removes  $T[u]$ from $T[a]$ and sets $u$ to be $v$'s other child.
\end{definition}

%Notice that the only edges separated by this action are those with exactly one endpoint in $\clust(u)$ and the other in $\clust(a) \setminus \clust(u)$. Previously, its coefficient was at least $n_T(p)$ (where $p$ is the parent of $u$), and now its coefficient is at most $n_T(v)$. We can use this to bound the operation cost.

%\begin{lemma}\label{lem:oprebalance}
%\lemoprebalance
%\end{lemma}

Tree rebalancing will only yield $1/6$-relatively balanced trees, which is interestingly something \citet{dasgupta}'s sparsest cut algorithm, one of the current best $\cost$ approximations, achieves via  a similar analysis. To refine this, we use subtree insertion and deletion (Figure~\ref{fig:ops}, top right). At a root with large child $a$ and small child $b$, we can move a small subtree from $a$ to $b$ to improve the balance.

\begin{definition}\label{def:delins}
Consider a binary tree $T$ with internal nodes $u$, some non-ancestor $v$, $u$'s sibling $s$, and $v$'s parent $g$. \textbf{Subtree deletion} at $u$ removes $T[u]$ from $T$ and contracts $s$ into its parent. \textbf{Subtree insertion} of $T[u]$ at $v$ inserts a new parent $p$ between $v$ and $g$ and adds $u$ as a second child of $p$. The operator $\opdelins(u,v)$ deletes $u$ and inserts $T[u]$ at $v$.
\end{definition}

%The only edges separated here are those with one endpoint in $\clust(u)$ and the other in $\clust(u\land v)\setminus \clust(u)$ (where we recall $u\land v$ is the least common ancestor of $u$ and $v$).

%\begin{lemma}\label{lem:opdelins}
%\lemopdelins
%\end{lemma}

We will also need to abstract away (Figure~\ref{fig:ops}, bottom left) certain levels of the hierarchy to simplify it. This involves taking vertices at depth $\depth_2$ and iteratively merging them into their parents until they reach depth $\depth_1$. In other words, ignore tree structure between two levels of the hierarchy.

\begin{definition}\label{def:abstract}
Consider a binary tree $T$ with two parameters $\depth_1$ and $\depth_2$ such that $\depth_1<\depth_2<\mathrm{height}(T)$. \textbf{Level abstraction} between levels $\depth_1+1$ and $\depth_2$ ($\opabs(\depth_1,\depth_2)$) involves taking all internal nodes between depths $d_1+1$ and $d_2$ in $T$ and contracting them into their parents. %The \textbf{distance} of the level abstraction is defined as the maximum ratio of $n_T(u)/n_T(v)$ of any two internal nodes $u$ and $v$ between depths $\depth_1$ and $\depth_2$ in $T$.% Level abstraction can also be defined in terms of height instead of depth, and the same properties apply.
\end{definition}

To achieve fairness in Section~\ref{sec:algs}, we use tree folding (Figure~\ref{fig:ops}, bottom right). Given multiple isomorphic trees (ignoring leaves), we map the topologies of the trees together.

\begin{definition}\label{def:fold}
Consider a set of subtrees $T_1,\ldots, T_k$ of $T$ such that all trees $T_i$ without their leaves have the same topology, and all $\rot(T_i)$ have the same parent $p$ in $T$. This means that for each $i\in[k]$, there is a tree isomorphism $\phi_i: I_i\to I_k$ where $I_i$ and $I_k$ are the internal nodes of the corresponding trees. A \textbf{tree folding} of trees $T_1,\ldots, T_k$ ($\opfold(T_1,\ldots,T_k)$) modifies $T$ such that all $T_1,\ldots,T_k$ are replaced by a single tree $T_f$ whose $\rot(T)$ is made a child of $p$ and $T_f$ has the same internal topology as $I_k$ such that for any leaf $\ell$ of any tree $T_i$ with parent $p_i$ in $T_i$, we set its parent to $\phi_i(p_i)$. %The \textbf{ratio} of the tree folding is the maximum ratio of $n_T(u)/n_T(v)$ for any $u\in T_i$ and $v\in T_j$ such that $\phi_i(u) = \phi_j(v)$.
\end{definition}

%We briefly note here (and discuss and prove in Section~\ref{sec:treeproof}) that tree rebalance length, subtree deletion/insertion shift, level abstraction distance, and tree rebalance ratio (times the number of involved trees) are the exact values of the proportionate increase of the $\cost$ of an edge when it is separated by such an operator. This is heavily used in our results.

%We use these operators as building blocks for our proceeding results. Each result in sequence introduces the addition of a new operator, however each result holds its own interest as well.
\section{Fair and Balanced Reductions}\label{sec:algs}

We now present our main algorithms, which sequentially build on top of each other, adding new operators in a limited, measureable capacity to achieve new results.

%All results achieve new hierarchical clustering qualities with limited applications of operations and limited increase in hierarchy cost. %Specifically, we achieve $1/6$-relative balance in Section~\ref{sec:rebalanced} by tree rebalances, $\epsilon$-relative balance in Section~\ref{sec:refined} by deletions and insertions, stochastic fairness in Section~\ref{sec:stochastic} by level abstraction, and deterministic fairness in Section~\ref{sec:fair} by tree folding.

%In Section~\ref{sec:rebalanced}, we propose a new algorithm for $1/6$-relatively balanced hierarchical clustering which uses tree rebalance operators (Definition~\ref{def:treer}). In Section~\ref{sec:refined}, we apply the subtree insertion and deletion operators (Definition~\ref{def:delins}) to achieve $\epsilon$-relative balance for any $\epsilon \in (0,1/6)$, where we may have $\epsilon = o(1)$. In Section~\ref{sec:stochastic}, we show how level abstraction (Definition~\ref{def:abstract}) can be used on top of a relatively balanced hierarchy ($1/6$-relative balance is sufficient) to yield a solution to the stochastic fair hierarchical clustering problem with bounded color probabilities. And finally, our main result can be found in Section~\ref{sec:fair}, where we solve our canonical fair hierarchical clustering problem with a tradeoff between fairness criteria and approximation factor.

\subsection{Relatively Rebalanced Trees}\label{sec:rebalanced}

Our first algorithm guarantees $1/6$-relative balance. It only modifies the tree through a series of limited tree rebalances (Definition~\ref{def:treer}). We can show that this only incurs a small constant factor proportionate increase in cost.

%\begin{theorem}\label{thm:balance}
%\thmbalance
%\end{theorem}

%\begin{customthm}{\ref{thm:balance}}
%\thmbalance
%\end{customthm}

\begin{theorem}\label{thm:balance}
\thmbalance
\end{theorem}

\codetreer

This idea was first introduced by \citet{dasgupta} as an analytical tool for their algorithm. However we use it more explicitly here to take any given hierarchy and rearrange it to be balanced. The basic idea is to start with some given tree $T$. Draw $T$ from top to bottom such that the smaller cluster in a split is put on the left. Let $A_1$ and $B_1$ be our first split. Continue working on the left side, splitting $A_1$ into $A_2$ and $B_2$ and so on. Stop when we find the first cluster $B_k$ such that $|B_k| \geq n/3$. This defines our first split: partition $V$ into $A_k$ and $B = \cup_{i=1}^k B_i$. Then recurse on each side.

It is not too hard to see that this yields a $\frac16$-relatively balanced tree. Our search stopping conditions enforce this.

%\drawrebalance

\begin{lemma}\label{lem:rebalance}
Algorithm~\ref{alg:rebalance} produces
a $\frac16$-relatively balanced tree.
\end{lemma}

The next property also comes from the fact that once an edge is separated, it will never be separated again.

\begin{lemma}\label{lem:balancesep}
In Algorithm~\ref{alg:rebalance}, every edge is separated by at most one tree rebalance operator.
\end{lemma}

Finally, the operation cost of the rebalance operators comes from our stopping threshold.

\begin{lemma}\label{lem:balancelength}
In Algorithm~\ref{alg:rebalance}, every tree rebalance operator has operation cost at most $3/2$.
\end{lemma}

In Theorem~\ref{thm:balance}, the relative balance comes from Lemma~\ref{lem:rebalance}, the operator properties come from Lemmas~\ref{lem:balancesep} and~\ref{lem:balancelength} respectively, and the approximation factor comes from  Lemma~\ref{lem:balancesep} and Lemma~\ref{lem:balancelength} together.

\subsection{Refining Relatively Rebalanced Trees}\label{sec:refined}

We now propose a significant extension of Algorithm~\ref{alg:rebalance} which allows us to get a stronger balance guarantee. Specifically (where $\epsilon$ may be a function of $n$):

%\begin{theorem}\label{thm:balance2}
%\thmbalancetwo
%\end{theorem}

\begin{theorem}\label{thm:balance2}
\thmbalancetwo
\end{theorem}

%\begin{customthm}{\ref{thm:balance2}}
%\thmbalancetwo
%\end{customthm}

To do this, we first apply $\treer$. Then, at each split starting at the root, we execute $\search$ (Algorithm~\ref{alg:search} in Appendix~\ref{apx:refine}), which searches for a small subtree below the right child, deletes it, and moves it below the left child in order to reduce the error in the relative balance. If we do this enough, we can reduce the relative balance to $\epsilon$. We call our algorithm, in Algorithm~\ref{alg:refinerebalance}, $\treerr$.

\coderefine

We can show that this algorithm creates a nicely rebalanced tree. $\search$ specifically guarantees a proportional $2/3$ reduction in relative balance (see Appendix~\ref{apx:refine}). Therefore, enough executions of $\search$ will make the split $\epsilon$-relatively balanced, and recursing down the tree guarantees that the entire tree is $\epsilon$-relatively balanced.

\begin{lemma}\label{lem:refinebalance}
Algorithm~\ref{alg:refinerebalance} produces an $\epsilon$-rebalanced tree.
\end{lemma}

To bound the operators on top of those used by Algorithm~\ref{alg:rebalance}, note that we only apply the subtree deletion and insertion operators. Additionally, each edge cannot be separated more than once.

\begin{lemma}\label{lem:refinesep}
In Algorithm~\ref{alg:refinerebalance}, every edge is separated by at most one subtree deletion and insertion operator.
\end{lemma}

Finally, we can also limit the operation cost of the subtree deletion and insertion operator. This is because we limit the depth of the $\search$ function as it will never be given a parameter below $\epsilon n$.

\begin{lemma}\label{lem:refineshift}
In Algorithm~\ref{alg:refinerebalance}, every subtree deletion and insertion operator has operation cost at most $3/\epsilon$.
\end{lemma}

For Theorem~\ref{thm:balance2}, the relative balance comes from Lemma~\ref{lem:refinebalance}, the operator properties com from Lemmas~\ref{lem:refinesep} and~\ref{lem:refineshift} respectively, and the approximation factor comes from Lemma~\ref{lem:opdelins}, Lemma~\ref{lem:refinesep}, and Lemma~\ref{lem:refineshift} together.

\subsection{Stochastically Fair Hierarchical Clustering}\label{sec:stochastic}

At this point, we almost have enough tools to solve stochastically fair hierarchical clustering. For this, however, we need a simple application of level abstraction (Definition~\ref{def:abstract}). We introduce $\treerm$, which simply imposes one level abstraction: $T' = T.\opabs(t,h_{max})$ on the bottom levels of the hierarchy. Here, $t$ is a parameter and $h_{max}$ is the max depth in $T$.
%It is modified such that, instead of considering vertices between a certain depth, it effectively searches from leaves to the root, and stops when it hits a cluster of a certain size, parameterized by $t$. It then contracts these paths for each leaf, excluding the leaf itself. For ease of discussion, we simply call this \textit{modified level abstraction}. Distance for modified level abstraction is also defined as the maximum cluster size ratio of two points involved in contraction.
Notice that we require the input to be relatively balanced to achieve this result.

%\begin{theorem}\label{thm:stoch}
%\thmstoch
%\end{theorem}

%\begin{customthm}{\ref{thm:stoch}}
%\thmstoch
%\end{customthm}

\begin{theorem}\label{thm:stoch}
\thmstoch
\end{theorem}

Theorem~\ref{thm:stoch} with constant $\alpha_\ell$ for all $\ell\in[\lambda]$, $c$, and $\delta$ yields a $O(\gamma\log n)$ approximation. Since $\gamma = O(\sqrt{\log n})$~\cite{charikar17,dasgupta}, this becomes $O(\log^{3/2} n)$. It is quite impressive, as the best previous fair (albeit, deterministic) approximation was $poly(n)$~\cite{ahmadian2020fairhier}. Also, $\delta$ exhibits an important tradeoff: increasing $\delta$ increases success probability but also decreases the range of acceptable $p_\ell(v)$ values.

It might be tempting to suggest applying $\treerm$ to any existing hierarchy, as opposed to one that is $\epsilon$-relatively balanced. However, if we consider, for instance, a highly unbalanced tree where all internal nodes have at least one leaf-child, the algorithm would only merge the bottom $t$ internal nodes into a single cluster, thereby not guaranteeing fairness. The resulting structure would also not be particularly interesting. This is why the rebalancing process is important.

Obviously, since we only apply  level abstraction once, edge separation only happens once per edge in the algorithm. To bound the operation cost, we explore the relative size of clusters at a specified depth in the hierarchy. The following guarantee is achieved by considering a root-to-vertex path where we always travel to maximally/minimally sized clusters according to the tree's relative balance.

\begin{lemma}\label{lem:balancedratio}
Let $T$ be an $\epsilon$-relatively balanced tree and $u$ and $v$ be internal nodes at depth $i$ in $T$. Then $(1/2-\epsilon)^in \leq n_T(u),n_T(v) \leq (1/2+\epsilon)^in$, which also implies $\frac{n_{T}(u)}{n_{T}(v)} \leq \frac{(1+2\epsilon)^{i}}{(1-2\epsilon)^{i}}$. Additionally, if $i \leq \log_{1/2-\epsilon}(x/n)$ for some arbitrary $x \geq 1$ and $\epsilon=1/(c\log_2 n)$ for a constant $c$, then the maximum cluster size is at most $ e^{4/(c(1-o(1)))}x$.
\end{lemma}

This yields our operation cost, since it bounds the size of clusters at certain depths.

\begin{lemma}\label{lem:stochdist}
In $\treerm$, the level abstraction has operation cost at most $(1/2+\epsilon)^tn$.
\end{lemma}

To get our fairness results, we need to use a Chernoff bound, thus we must guarantee that all internal nodes have sufficiently large size. This too comes from our bounds on cluster sizes.

\begin{lemma}\label{lem:stochsize}
In $\treerm$, for any internal node $v$, $n_{T'}(v) \geq (1/2-\epsilon)^tn$.
\end{lemma}

Finally, we must show the fairness guarantee. Since the union of two fair clusters is fair, we only need to show this for the clusters at height 1 in the hierarchy, as this would imply fairness for the rest of the hierarchy. This comes from a Chernoff bound.

\begin{lemma}\label{lem:stochchernoff}
The resulting tree from $\treerm$ with $t = \log_{1/2-\epsilon}\left(\frac{3(1-\delta)\ln(cn)}{a\delta^2n}\right)$ for $a = \min_{\ell\in[\lambda]}\alpha_\ell$ and any $\delta>0$ is stochastically fair for given parameters $\alpha_\ell,\beta_\ell$ for all colors $\ell\in [\lambda]$ with high probability if with $\frac{1}{1-\delta}\alpha_\ell \leq p_\ell(v) \leq \frac1{1+\delta}\beta_\ell$ for all $v\in V$ and $\ell\in [\lambda]$ for $\lambda = O(1)$.
\end{lemma}

This is sufficient to show Theorem~\ref{thm:stoch}. The fairness is a result of Lemma~\ref{lem:stochchernoff}, the operator properties are a result of Lemma~\ref{lem:stochdist} and the obvious fact that we only apply one level abstraction, and the approximation factor comes from Lemma~\ref{lem:opabstract} and Lemma~\ref{lem:stochdist} together.

\subsection{Deterministically Fair Hierarchical Clustering}\label{sec:fair}

Finally, we have our main results on the standard, deterministic fair hierarchical clustering problem. This algorithm builds on top of the results from Theorem~\ref{thm:balance2} and uses methods similar to Theorem~\ref{thm:stoch}. In addition to previous algorithms, it uses more applications of level abstraction and introduces tree folding. %Our final result is achieved by running: 1) a $\gamma$-approximate (unfair) hierarchical clustering algorithm, 2) Algorithm~\ref{alg:rebalance}, 3) Algorithm~\ref{alg:refinerebalance} with $\epsilon=1/(c\log n)$, and 4) Algorithm~\ref{alg:fhc}.

\begin{theorem}\label{thm:main}
\thmmain

This algorithm runs in $O(n^2\log n)$ time.
\end{theorem}

%\begin{customthm}{\ref{thm:main}}
%\thmmain
%\end{customthm}

Since $\gamma = O(\sqrt{\log n})$, this becomes $O(n^\delta\log^{5/2}n)$ for any constant $c$, $\delta\in(0,1)$, and $k$ which greatly improves the previous $O(n^{5/6}\log^{5/4}(n))$-approximation~\cite{ahmadian2020fairhier}. Additionally, the previous work only considered 2 colors with equal representation in the dataset. Our algorithm greatly generalizes this to both more colors and different proportions of representation. While we do not guarantee exact color ratio preservation as the previous work does, our algorithms can get arbitrarily close through parameterization and we no longer require the ratio between colors points in the input to be exactly 1.

In terms of fairness, all of the variables here are parameterizeable constants. Increasing $k$, $c$, and $\delta$ will all make these values get closer to the true proportions of the colors in the overall dataset, and this can be done to an arbitrary extent. Therefore, based off the parameterization, this allows us to enforce clusters to have pretty close to the same color proportions as the underlying dataset.

%Notice that after applying the results of Theorem~\ref{thm:stoch}, we can assume our tree is a binary tree with $\epsilon$-relative balance up until depth $\log_{1/2-\epsilon}(1/(2n\epsilon))$, at which point internal nodes immediately split into only leaf children.
%Notice that at this depth, Lemma~\ref{lem:balancedratio} yields that the smallest cluster size is at least $1/(2\epsilon)$. This is precisely the point at which the $\epsilon$-relative balance loses its guarantee. So for every split above the bottom layer in the tree, the $\epsilon$-relative balance is meaningful.
The goal of this algorithm is to recursively abstract away the top $\log_2h$ depth of the tree, where we end up setting $h = n^\delta$. Each time we do this, we get a kind of ``frontier clustering'', which is an $h$-sized clustering whose parents in the tree are all the root after level abstraction. Since the subtrees rooted at each cluster have the same topology (besides their leaves, this is due to our level abstraction at the lowest levels in the tree), we can then execute tree folding on any subset of them. We select cluster subtrees to fold together such that, once we merge the appropriate clusters, the clustering at this level will be more fair. Then, as we recurse down the tree, we subsequently either eliminate clusters (via level abstraction) or fold them to guarantee fairness. For more information, see Algorithm~\ref{alg:fhc}.

\codefhc

%\codefold

%We now discuss our result in Algorithm~\ref{alg:fhc}. The idea behind this is at the root of the tree, we execute Algorithm~\ref{alg:fold}. This creates level abstraction for the top $\log_2h$ levels. This yields a single $h$-clustering that connects directly to the root in the resulting tree. We then partition the vertices corresponding to the clustering into $k$-sized parts in a way such that if we merge these clusters together, they provide us with fairness guarantees. On each part, we do a tree fold on all $k$ trees. Then, our main Algorithm~\ref{alg:fhc} recurses on each subtree.

To see why this creates a fair, low-cost hierarchy, we first bound the metrics on the operators used. When we execute level abstraction, we can leverage relative balance and Lemma~\ref{lem:balancedratio} to show that during $\fairhc$, we can bound the abstraction operation cost.

\begin{lemma}\label{lem:folddist}
In Algorithm~\ref{alg:fhc}, the level abstraction has operation cost at most $e^{2/c}h$.
\end{lemma}

Our tree folding operation cost bound also comes from the balance of a tree, since any two vertices that are folded together must be at similar depths.

\begin{lemma}\label{lem:foldratio}
In Algorithm~\ref{alg:fhc}, each tree folding has operation cost at most $ke^{4/(c(1-o(1))}$ and acts on $k$ trees.
\end{lemma}

In order to bound the cost, we need to first know how many times an edge will be separated. We notice that an edge that is separated by level abstraction can no longer be separated on a subsequent recursive step. Additionally, the number of tree fold operators is proportionally bounded by the recursive depth, as it only happens $\lambda$ times each step.

\begin{lemma}\label{lem:foldnum}
In Algorithm~\ref{alg:fhc}, an edge $e$ is separated by at most 1 level abstraction and $\lambda\log_2(n)/\log_2(h)$ tree folds. The maximum recursion depth is also at most $\log_2(n)/\log_2(h)$.
\end{lemma}

Fairness comes from the ordering over $\ell$-colored vertices and the way select subtrees to fold together. One recursive step of $\fairhc$ incurs a small constant factor proportionate loss in potential fairness, and the number of times this loss occurs is bounded by the depth of recursion. We desire these fractions to be close to the true color proportions, which we can get arbitrarily close to by setting parameters $c$, $k$, and $h$.

\begin{lemma}~\label{lem:fair}
For an $\epsilon$-relatively balanced hierarchy $T$ over $\ell(V) = c_\ell n = O(n)$ vertices of each color $\ell\in[\lambda]$, Algorithm~\ref{alg:fhc} yields a hierarchical clustering $T'$ such that the amount of each color $\ell\in[\lambda]$ in each cluster (represented by vertex $v$) is bounded as follows:

\[\frac{c_\ell}{e^{2\log_hn/c}}\leq \frac{\ell(v)}{n_T(v)} \leq c_\ell\cdot ( e^{4/c}/(kc_\ell) + e^{6/c})^{\log_hn}.\]
\end{lemma}

In Theorem~\ref{thm:main}, fairness is a result of Lemma~\ref{lem:fair}, the operator properties are a result of Lemmas~\ref{lem:folddist},~\ref{lem:foldratio}, and~\ref{lem:foldnum}, and the approximation factor has already been worked out by Lemma~\ref{lem:foldapx}.

\section{Experiments}
This section validates our algorithms from Section~\ref{sec:algs}. Our simulations demonstrate that our algorithm incurs only a modest loss in the hierarchical clustering objective and exhibits increased fairness. Specifically, the approximate cost increases as a function of Algorithm \ref{alg:fhc}'s defining parameters: $c, \delta, $ and $k$. 

\paragraph{Datasets.} We use two data sets, \emph{Census} and \emph{Bank}, from the UCI data repository \cite{Dua2019}. Within each, we subsample only the features with numerical values. To compute the \emph{cost} of a hierarchical clustering we set the similarity to be $w(i,j) = \frac{1}{1 + d(i,j)}$ where $d(i,j)$ is the Euclidean distance between points $i$ and $j$. We color data based on binary (represented as blue and red) protected features: \emph{race} for \emph{Census} and \emph{marital status} for \emph{Bank} (both in line with the prior work of \citet{ahmadian2020fairhier}). As a result, \emph{Census} has a blue to red ratio of 1:7 while \emph{Bank} has 1:3.

We then subsample each color in each data set such that we retain (approximately) the data's original balance. We use samples of size 256. For each experiment, we do 10 replications and report the average results. We vary the parameters $c \in \{2^i\}_{i=0}^5, \delta \in (\frac18, \frac78)$, and $k \in \{2^i\}_{i=1}^4$ to experimentally validate their theoretical impact on the approximate guarantees of Section~\ref{sec:algs}.

\paragraph{Implementation.} The Python code for the following experiments are available in the Supplementary Material. We start by running average-linkage, a popular hierarchical clustering algorithm. We then apply Algorithms~\ref{alg:rebalance} - \ref{alg:fhc} to modify this structure and induce a \emph{fair} hierarchical clustering that exhibits a mild increase in the cost objective. 

\paragraph{Metrics.} In our results we track the approximate cost objective increase as follows: Let $G$ be our given graph, $T$ be average-linkage's output, and $T'$ be Algorithm~\ref{alg:fhc}'s output. We then measure the ratio $\textsc{Ratio}_{cost} = \frac{cost_G(T')}{cost_G(T)}$. %To evaluate fairness, we plot the balance exhibited by each cluster across all clusters in the hierarchy. This is done on a single trial for illustrative purposes.

\paragraph{Results.} We first note that the average-linkage algorithm must construct unfair trees since, for each data set, the algorithm induces some monochromatic clusters. Thus, our resultant fair clustering is of considerable value in practice.

\begin{figure}[t]%[!htb]
	\center{\includegraphics[width=0.95\linewidth]{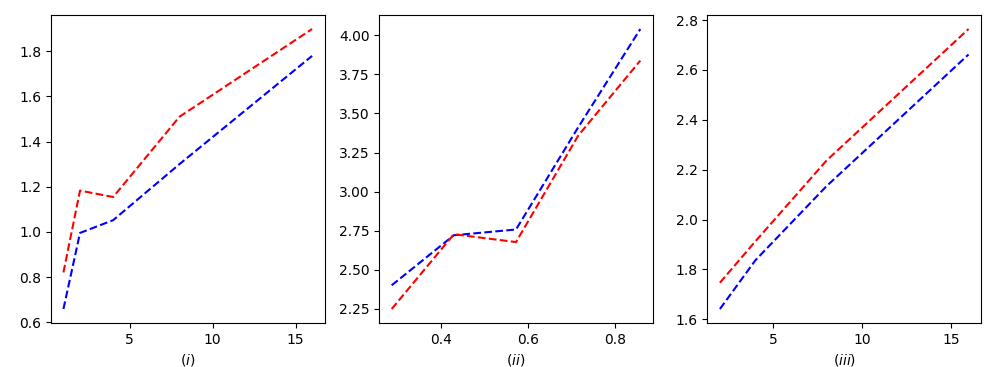}}
    \caption{\label{fig:n128_params} Cost ratio of Algorithm~\ref{alg:fhc} as compared to average-linkage. (i) Ratio increase as a function of the parameter $c$, (ii) ratio increase as a function of the parameter $\delta$, and (iii) ratio increase as a function of $k$. Blue lines indicate the result for \emph{Census} dataset whereas red indicates the \emph{Bank} dataset results.
    }

\vspace{-3mm}
\end{figure}

In Figure 1, we plot the change in cost ratio as the parameters ($c, \delta, k$) are varied for the two datasets. Supporting our theoretical results, increasing our fairness parameters leads to a modest increase in $\cost$. This is an empirical illustration of our fairness-$\cost$ approximation tradeoff according to our parameterization. Note that the results are consistent across tested datasets.

%We note that the plotted results are restricted to the Census data set, with the comparable results on the Bank data set presented in Appendix \ref{} due to space constraints. 
%Figure 1(i), (ii), and (iii) depict $\textsc{ratio}_{\textsc{cost}}$ as a function of parameters $c$ (which dictates $\epsilon$ in the analysis of Section~\ref{sec:algs}), $\delta$, and $k$ respectively. The plots demonstrate that manipulating the original hierarchy to produce a fair clustering of the data yields an increase in the cost objective.

We additionally illustrate the resulting balance of our hierarchical clustering algorithm by presenting the distribution of the cluster ratios of the projected features (blue to red points) in Figure~\ref{fig:fairness} for the \emph{Census} data. The output of average-linkage naturally yields an unfair clustering of the data, yet after applying our algorithm on top this hierarchy we see that the cluster's balance move to concentrate about the underlying data balance of 1:7. An equivalent figure for the \emph{Bank} dataset is provided in the appendix due to space constraints.

\begin{figure}[h]%[!htb]
	\center{\includegraphics[width=0.95\linewidth]{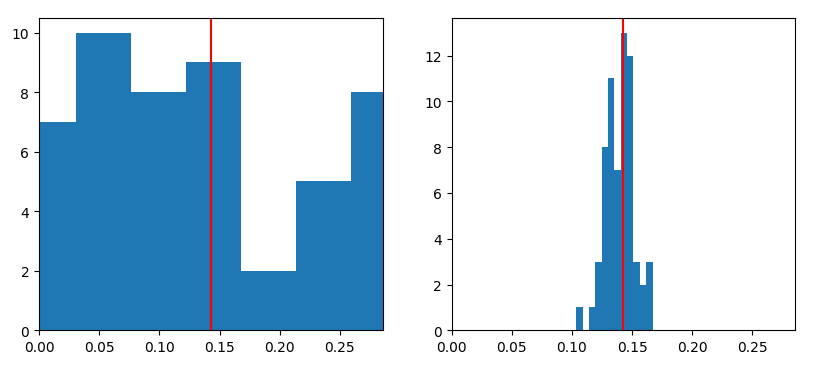}}
    \caption{\label{fig:fairness} Histogram of cluster balances after tree manipulation by Algorithm~\ref{alg:fhc}. The left plot depicts the balances after applying the average-  linkage algorithm and the right shows the result of applying our algorithm. The vertical red line indicates the balance of the dataset. Parameters were set to $c = 4, \delta = \frac38, k = 4$ for the above clustering result.
    }

\vspace{-3mm}
\end{figure}

% In the unusual situation where you want a paper to appear in the
% references without citing it in the main text, use \nocite
%\nocite{langley00}

\bibliography{references}
\bibliographystyle{icml2023}

%%%%%%%%%%%%%%%%%%%%%%%%%%%%%%%%%%%%%%%%%%%%%%%%%%%%%%%%%%%%%%%%%%%%%%%%%%%%%%%
%%%%%%%%%%%%%%%%%%%%%%%%%%%%%%%%%%%%%%%%%%%%%%%%%%%%%%%%%%%%%%%%%%%%%%%%%%%%%%%
% APPENDIX
%%%%%%%%%%%%%%%%%%%%%%%%%%%%%%%%%%%%%%%%%%%%%%%%%%%%%%%%%%%%%%%%%%%%%%%%%%%%%%%
%%%%%%%%%%%%%%%%%%%%%%%%%%%%%%%%%%%%%%%%%%%%%%%%%%%%%%%%%%%%%%%%%%%%%%%%%%%%%%%
\clearpage
\newpage
\appendix
%\onecolumn

\section{Limitations}\label{apx:limitations}

The main limitations this work suffers from encapsulate the general limitations of study in theoretical clustering fairness. Our work strives to provide algorithms that are applicable to many hierarchical clustering applications where fairness is a concern. However, our work is inherently limited by its focus on a specific fairness constraint (i.e., the extension of disparate impact originally used to study fair clustering~\cite{chierichetti2017fair}). While disparate impact has received substantive attention in the clustering community and is seen as one of the primary fairness definitions/constraints~\citep[see, e.g.,][]{ahmadian2020fairhier,ahmadian2020faircorr,bera2019fair,brubach2020a,kleindessner}, it is just one of many established fairness constraints for problems in clustering~\cite{chakrabarti2021a,proportionalfairness,esmaeili2021fair,kleindessner2}. When applying fair machine learning algorithms to problems, it is not always clear which fairness constraints are the best for the application. This, and the fact that the application of fairness to a problem can cause harm in other ways~\cite{benporat2021protecting}, means that the proposal of theoretical fair machine learning algorithms always has the potential for improper or even harmful use. While this work proposes purely theoretical advances to the field, we direct the reader to~\citep{hardtbook} for a broader view on the field.

Our results are also limited by the theoretical assumptions that we make. For instance, in the stochastic fairness algorithm, we assume that the probabilities of a vertex being a certain color are within the same bounds across all vertices. This may not be realistic, as there could be higher variance in the distribution of color probabilities, and even though the probabilities may lie outside of our assumed bounds, it still may be tractable to find a low-$\cost$ hierarchical clustering.

In our main theorem, we assume that there are only two colors (protected classes), and that they subsume a constant fraction of the general population. The former assumption is clearly limited in that in many cases, protected classes may take on more than two values.  The constant fraction assumption is actually highly relevant and is reflected in other clustering literature, but it is a potential limitation that may rule out a handful of applications nevertheless.

Finally, our results are limited to the evaluation of hierarchical clustering quality based off $\cost$. While this is a highly regarded metric for hierarchy evaluation, there may be situations where others are appropriate. It also neglects the practicality of empirical study in that many important machine learning algorithms we use today cannot provide guarantees across all data (which our results necessarily do), but they perform much better on most actual inputs. However, we leave it as an open question to further evaluate the practicality of our algorithms through empirical study.
\section{Proofs: Tree Properties and Operators}\label{sec:treeproof}

Here we present all our proofs and theoretical results regarding our tree operator properties.

We start by discussing our tree rebalance operator.  Effectively, any edge whose end points are separated by a tree rebalance operator was contained in a cluster of size $n_T(u)$, and now we guarantee they are in a cluster of size $n_T(v)$.

%\begin{lemma}\label{lem:oprebalance}
%Say we apply a tree rebalance operator of length $\ell$ on hierarchy $T$ to yield $T'$. For any edge $e$ whose endpoints are separated by the operator, $\cost_{T'}(e) \leq \ell\cost_{T}(e)$. For every other edge $e$, $\cost_{T'}(e) \leq \cost_T(e)$.
%\end{lemma}

%\begin{customlem}{\ref{lem:oprebalance}}
%\lemoprebalance
%\end{customlem}

\begin{lemma}\label{lem:oprebalance}
\lemoprebalance
\end{lemma}

\prflemoprebalance

For our subtree deletion and insertion, the idea is that an edge that is separated costs at least $n_T(v)$ in the original tree, but may cost up to $n_T(u\land v)$ in the modified tree.

%\begin{customlem}{\ref{lem:opdelins}}
%\lemopdelins
%\end{customlem}

\begin{lemma}\label{lem:opdelins}
\lemopdelins
\end{lemma}

\prfopdelins

The level abstraction operator is somewhat more complicated, as it modifies entire levels of the tree, instead of individual splits. However, we can still use our notion of operation cost to bound the operator's impact. This just becomes a bit more vague because we have to look at the largest and smallest clusters between depths $h_1$ and $h_2$ in $T$.

\begin{lemma}\label{lem:opabstract}
Say we apply the level abstraction operator between heights $h_1$ and $h_2$ on hierarchy $T$ to yield $T'$. An edge is separated by the operator if and only if the least common ancestor of its endpoints is between $h_1$ and $h_2$. Its operation cost is at most $\Delta \leq  \frac{n_T(u)}{n_T(v)}$, where $u$ and $v$ are two clusters that are abstracted away that maximize this ratio.

%For any edge $e$ whose endpoints are separated by the operator, $\cost_{T'}(e) \leq d\cost_{T}(e)$. For every other edge $e$, $\cost_{T'}(e) \leq \cost_T(e)$.
\end{lemma}

\prfabstract

Tree folding is a bit more complicated because we are merging multiple clusters on top of each other. Thus we have to factor in the value $k$ on top of considering varying cluster sizes. Ultimately, however, the product of the ratio between cluster size and $k$ bound the proportional increase in cost.

\begin{lemma}\label{lem:opfold}
Say we apply the tree folding operator on hierarchy $T$ to yield $T'$. Its operation cost is at most $\Delta \leq  k\frac{n_T(u)}{n_T(v)}$, where $u$ and $v$ are two clusters that are mapped to each other away that maximize this ratio.

%For any edge $e$ whose endpoints are separated by the operator, $\cost_{T'}(e) \leq \rho k\cost_{T}(e)$. For every other edge $e$, $\cost_{T'}(e) \leq \cost_T(e)$.
\end{lemma}

\prffold
\section{Proofs: Results}

In this section, we prove all lemmas, theorems, and missing algorithmic discussion regarding our main results.

\subsection{$\treer$}

This section contains the proofs regarding $\treer$.

\begin{proof}[Proof of Lemma~\ref{lem:rebalance}]
By our definition of $A_i$ and $B_i$ for all $i\in [k]$, $|A_{k-1}| \geq 2n/3$, implying $|A_k| \geq \frac12|A_{k-1}| \geq n/3$, and also that $|A_k| \leq n/3$ since $|B_k| \geq n/3$. Thus $n/3 \leq |A_k|, |B| \leq 2n/3$. Since we rearrange our first split to be this way, that means our first tree rebalance creates a first split that satisfies the relatively balanced condition. From here, we recurse on each side, guaranteeing that one split after another satisfies the condition. Thus, the entire tree is $\frac16$-relatively balanced.
\end{proof}

\begin{proof}[Proof of Lemma~\ref{lem:balancesep}]
Consider an edge $e=(x,y)$ that is first rebalanced at some recursive step in Algorithm~\ref{alg:rebalance}. By Lemma~\ref{lem:oprebalance}, $x$ and $y$ must now be separated at the current tree's root. Therefore, at any further level of recursion, only one of $e$'s endpoints will be present, so it cannot be separated again.
\end{proof}

\begin{proof}[Proof of Lemma~\ref{lem:balancelength}]
The rebalance operator is applied to $v$ (the node found) at $r$. Notice that $v$'s parent $p$ must be such that $n_T(p) \geq \frac23n_T(r)$, otherwise the loop would have stopped earlier. Therefore, by Lemma~\ref{lem:oprebalance}, the operation cost is $n_T(r)/n_T(p) \leq 3/2$.
\end{proof}

\begin{proof}[Proof of Theorem~\ref{thm:balance}]
Let $T^*$ be the optimal tree, let $T_1$ be our guaranteed $\gamma$-approximation on $T$, and let $T'$ be our output. By Lemma~\ref{lem:rebalance}, $T'$ is $1/6$-relatively balanced. By Lemmas~\ref{lem:balancesep} and~\ref{lem:balancelength}, every edge is separated by at most one tree rebalance operator of length at most $3/2$. Because of this, $\cost_{T'}(e) \leq (3/2) \cost_{T_1}(e)$. Summing over all edges yields $\cost(T) \leq (3/2)\cost(T_1) \leq \gamma\cost(T^*)$.
\end{proof}

\subsection{$\treerr$}\label{apx:refine}

This section contains the proofs regarding $\treerr$ as well as the algorithmic description of $\search$.

As discussed in the body, at a given split, $\search$ traverse the tree below the larger cluster further down in a similar manner until we find a sufficiently small cluster. This cluster must be smaller than the current balance error, $\epsilon$. We simply do this by always traversing to the larger cluster as in Algorithm~\ref{alg:search} until its smaller child is sufficiently small. We then remove that subtree, traverse back to the top of the tree, and try to reinsert the subtree by recursing down the right children.

\codesearch

This exhibits nice properties with respect to relative balance.

\begin{lemma}\label{lem:searchsix}
$\search$ preserves $\frac16$-relative balance.
\end{lemma}

\begin{proof}[Proof of Lemma~\ref{lem:searchsix}]
Consider $T$, the tree at the beginning of the algorithm, and let $v$ be the vertex whose subtree we end up moving. To start, we only consider the deletion, and then we will consider the reinsertion of $v$'s subtree. The only vertices whose corresponding cluster sizes are altered (specifically, reduced) are $v$'s ancestors. Note that they are all right children (i.e., the bigger sibling at the start) and they are reduced by size $n_T(v)$.

Let $p$ be the parent of $v$. Since $v = \lft_T(p)$, we know $n_T(v) \leq \frac12n_T(p)$. Since we remove that many vertices, $n_T(p)$ is at worst halved. Since $p$ is a right child, say with sibling node $q$, $n_T(p) \geq n_T(q)$ at the start. Then at the end, $n_{T'}(p) \geq \frac12n_{T'}(q)$. This implies that, in the end, the clusters are between 1/3 and 2/3 the size of their parent. Thus relative balance is held on this split. For ancestor nodes $a$ of $p$ in $T$, this argument holds since $n_T(a) > n_T(p)$ both before and after, and $a$ is also a right child. Therefore, the entire tree is still $\frac16$-relatively balanced after subtree deletion.

Now we consider the second half of the algorithm, where we reinsert $T[v]$. Let $u$ be the vertex we select to insert at, $p$ be its new parent, $g$ be its old parent (now its grandparent), and $r = \rgt_T(g)$ be its old sibling. Before insertion, we know that $n_T(r) \geq n_T(v)$ by the while loop condition. Since $T$ is $\frac16$-relatively balanced still, and $r$ is $u$'s sibling, $n_T(u) \geq \frac12n_T(r)$. Since the algorithm did not stop at $p$, then $n_T(r) \geq n_T(v)$, thus implying $n_T(u) \geq \frac12 n_T(v)$. Additionally, since the algorithm stopped on $u$, $n_T(\rgt_T(u)) \leq n_T(v)$. Since that is $u$'s larger child, $n_T(u) \leq 2n_T(\rgt_T(u)) \leq 2n_T(v)$. Since $v$ is $u$'s new sibling, and one is not more than twice the size of the other, we have $\frac16$-relative balance at that split.

The only other vertices impacted by the insertion are $u$'s ancestors. For an ancestor node $a$ of $u$ in $T$, the argument also holds since $n_{T}(a) \geq n_T(p) \geq n_T(v)$ meaning $n_{T'}(a) \leq 2n_T(a)$ and $a$ must be a left (and therefore smaller) child. Therefore, the relative balance is kept at all splits involving ancestors of $u$, thus we have relative balance.
\end{proof}

Our other guarantee is that we find a subtree of size at least $s/2$ to move. This comes from our first loop's end condition.

\begin{lemma}\label{lem:searchhalf}
In Algorithm~\ref{alg:search}, $s/3 \leq n_T(v) \leq s$.
\end{lemma}

\begin{proof}[Proof of Lemma~\ref{lem:searchhalf}]
When the first loop stops, this is the first visited vertex whose left child, which ends up being the final $v$, is at most $s$. Thus $n_T(v) \leq s$. Since this was the first such instance, if $g$ is the grandparent of $v$, this means $n_T(\lft_T(g)) > s$ since the loop continued after $g$. Since right children are larger and $v$'s parent $p$ is $\rgt_T(g)$, $n_T(p) \geq n_T(\lft_T(g)) > s$. Since we have $\frac16$-relative balance, $n_T(v) \geq \frac13 n_T(p) \geq \frac13s$.
\end{proof}

\begin{proof}[Proof of Lemma~\ref{lem:refinebalance}]
At each iteration of Algorithm~\ref{alg:refinerebalance}, as long as the relative balance is above $\epsilon$, we move a subtree of size at least $\frac13\delta n$ and at most $\delta n$ by Lemma~\ref{lem:searchhalf} where $\delta$ is the current relative balance. This means that the relative balance of the first split reduces by a factor of $\frac23$, and by Lemma~\ref{lem:searchsix}, the rest of the tree remains $\frac16$-relatively balanced. This is simply done until the relative balance of the first split is small enough. When we recurse, we are still guaranteed $\frac16$-relatively balance, and we can then ensure all sufficiently large splits are $\epsilon$-relatively balanced.
\end{proof}

\begin{proof}[Proof of Lemma~\ref{lem:refinesep}]
Consider an edge $e$ that is first separated by some subtree deletion and insertion operator at some recursive step in Algorithm~\ref{alg:refinerebalance}. Notice $e$ must now be separated at the current tree's root. This means that at any further level of recursion, only one of $e$'s end points will be present, so it cannot be separated again.
\end{proof}

\begin{proof}[Proof of Lemma~\ref{lem:refineshift}]
The subtree deletion and insertion operator is applied at $u$ of $T[v]$ when $u\land v$ is the root, i.e., $n_T(u\land v) = n_T(r) \leq n$ where $r$ is the current tree's root and $n$ is our original data set size. We never allow the algorithm to continue with $\delta\leq \epsilon$, therefore the smallest tree size $T[v]$ that we move is $\frac13n\epsilon$ by Lemma~\ref{lem:searchhalf}. Thus the operation cost is at most $n_T(u\land v)/n_T(v) \leq \frac3{\epsilon}$ by Lemma~\ref{lem:opdelins}.
\end{proof}

\begin{proof}[Proof of Theorem~\ref{thm:balance2}]
Let $T^*$ be the optimal tree, let $T_1$ be our $1/6$-relatively balanced $3\gamma/2$ approximation guaranteed by Theorem~\ref{thm:balance}, and let $T'$ be our output. By Lemma~\ref{lem:refinebalance}, $T'$ is $\epsilon$-relatively balanced. By Lemmas~\ref{lem:refinesep} and~\ref{lem:refineshift}, every edge is separated by at most one subtree deletion and insertion operator of operation cost at most $3/\epsilon$. Because of this and because of Lemma~\ref{lem:opdelins}, $\cost_{T'}(e) \leq \frac3{\epsilon}\cost_{T_1}(e)$. Summing over all edges yields $\cost(T) \leq \frac3{\epsilon}\cost(T_1) \leq \frac{9\gamma}{2\epsilon}\cost(T^*)$.
\end{proof}

\subsection{$\treerm$}

This section contains the proofs regarding $\treerm$.

\begin{proof}[Proof of Lemma~\ref{lem:balancedratio}]
Since $T$ is $\epsilon$-relatively balanced, any cluster $A$ that splits into clusters $B$ and $C$ satisfies $(1/2-\epsilon)|A| \leq |C| \leq |B| \leq (1/2+\epsilon)|A|$, without loss of generality. This means that the maximum cluster size that can be found at level $i$ is bounded above by traversing the tree from root down assuming that we always traverse to a maximally sized child, e.g., if $p$ is a parent of $w$ on our path, then $n_T(w) \leq (1/2 + \epsilon)n_T(p)$.

Since we traverse $i$ levels, we get for any $i$-level vertex $u$, $n_T(u) \leq (1/2 + \epsilon)^in$. By the reverse logic (i.e., traversing from the root to a minimally sized child), for any $i$-level vertex $v$, $n_T(v) \geq (1/2-\epsilon)^in$. Then their ratio must be at most $\frac{n_T(u)}{n_T(v)} \leq \frac{(1+2\epsilon)^{i}}{(1-2\epsilon)^{i}}$.

Finally, consider if $i \leq \log_{1/2-\epsilon}(x/n)$. We can just assume it is at the maximum possible level, because this will clearly give the loosest bounds. We already know the smallest cluster size at level $i$ is at least $(1/2-\epsilon)^in$, and the ratio between the largest and smallest cluster sizes is at most $\left(\frac{1+2\epsilon}{1-2\epsilon}\right)^i$. Therefore, for a vertex $u$ at level $i$:

\[n_T(u) \leq \left(\frac{1+2\epsilon}{1-2\epsilon}\right)^{\log_{1/2-\epsilon}(n/x)}(1/2-\epsilon)^{\log_{1/2-\epsilon}(x/n)}n\]

The second term in the product obviously simplifies to $x$. For the first term, we can see that since $\epsilon = 1/(c\log_2 n)$:

\[\frac{1+2\epsilon}{1-2\epsilon} = 1 + \frac{4\epsilon}{1-2\epsilon} = 1 + \frac{4}{c(1-2\epsilon)\log_2n}\]

We can also bound the exponent. Note that we raise a value that is at least 1 to the exponent, so to create an upper bound, we must upper bound the exponent as well. We leverage the fact that $1/(1/2-\epsilon) >2$ because $\epsilon \in(0,1/2)$. This implies $\log_2(1/(1/2-\epsilon)) > 1$.

\begin{align*}
\log_{1/2-\epsilon}(x/n) =& \frac{\log_2(x/n)}{\log_2(1/2-\epsilon)} 
\\=&\frac{\log_2(n/x)}{\log_2(1/(1/2-\epsilon))} 
\\\leq& \log_2(n/x)
\\\leq& \log_2(n)
\end{align*}

Where the last step comes from the fact taht $x \geq 1$. We can now put this all together.

\[n_T(u) \leq \left(1 + \frac{4}{c(1-2\epsilon)\log_2n}\right)^{\log_2(n)} x\leq x\cdot e^{4/(c(1-o(1))} \]
\end{proof}

\begin{proof}[Proof of Lemma~\ref{lem:stochdist}]
By Lemma~\ref{lem:balancedratio}, the smallest cluster at depth $i \geq t$ is at most $(1/2+\epsilon)^tn$. If we assume a trivial cluster size lower bound of 1, this implies for any contracted internal nodes $u$ and $v$ in the level abstraction, $n_T(u)/n_T(v) \leq (1/2+\epsilon)^tn$.
\end{proof}

\begin{proof}[Proof of Lemma~\ref{lem:stochsize}]
By Lemma~\ref{lem:balancedratio}, the largest cluster at depth $i \leq t$ in $T$ is at least $(1/2-\epsilon)^tn$. When we execute level abstraction, this cluster size is not changed, but we know all other potentially smaller clusters are contracted into their parents. Thus, this is the smallest cluster size in $T'$.
\end{proof}

\begin{proof}[Proof of Lemma~\ref{lem:stochchernoff}]
Consider a vertex $v$ at height 1. By Lemma~\ref{lem:stochsize}, $n_{T'}(v) \geq (1/2-\epsilon)^tn = \frac{3(1-\delta)}{a\delta^2}\ln(cn)$. Fix some $\ell\in[\lambda]$. Let $X_{\ell v}$ count the number of vertices of color $\ell$ in $\leaves(v)$. Note this is a sum of Bernoullis, so $\mathbb{E}[|X_{\ell v}|] = \sum_{u\in\clust(v)} p_\ell(u)$. Note that we are given that $p_\ell(u) \geq \frac{1}{1-\delta}\alpha_\ell$ for all $u$. This gives us the following bounds from Lemma~\ref{lem:stochsize}:

\begin{align*}
\mathbb{E}[X_{\ell v}] \geq& \frac{3(1-\delta)}{a\delta^2}\ln (cn)\cdot \frac{1}{1-\delta}a = \frac{3}{\delta^2}\ln (cn)
\\\mathbb{E}[X_{\ell v}]  \geq& \frac{1}{1-\delta}\alpha_\ell n_{T'}(v)
\\\mathbb{E}[X_{\ell v}]  \leq& \frac{1}{1+\delta}\beta_\ell n_{T'}(v)
\end{align*}

Then by a Chernoff bound with $\delta$ as the error parameter:

\begin{align*}
P(|X_{\ell v} - \mathbb{E}[X_{\ell v}]| &\geq \delta \mathbb{E}[X_{\ell v}]) \\\leq& 2\exp(-\mathbb{E}[X_{\ell v}]\delta^2/3)
\\=& 2\exp(-\frac{3}{\delta^2}\ln (cn)\delta^2/3)
\\=& \frac{2}{cn}
\end{align*}

Thus with probability at least $1-\frac{2}{cn}$:

\begin{align*}
X_{\ell v} - \mathbb{E}[X_{\ell v}] \leq& \delta \mathbb{E}[X_{\ell v}]
\\X_{\ell v} \leq& (1+\delta)\mathbb{E}[X_{\ell v}]
\\\leq& (1+ \delta)\cdot \frac{1}{1+\delta} \beta_\ell n_{T'}(v)
\\\leq& \beta_\ell n_{T'}(v)
\end{align*}

In other words, the cluster $\leaves(v)$ satisfies the upper bound for color $\ell$.  We also find that:

\begin{align*}
-X_{\ell v} + \mathbb{E}[X_{\ell v}] \geq& \delta \mathbb{E}[X_{\ell v}]
\\X_{\ell v} \geq& (1-\delta)\mathbb{E}[X_{\ell v}]
\\\geq& (1-\delta)\cdot \frac{1}{1-\delta} \alpha_\ell n_{T'}(v)
\\\geq& \alpha_\ell n_{T'}(v)
\end{align*}
Which means it also satisfies the upper bound. Let $y$ be the number of internal nodes with leaf-children. Since we already saw the minimum such cluster size is $O(\log n)$ (since $a,\delta = O(1)$), then $y = O(n/\log n)$. Notice, also, that the vertices counted by $y$ are the only ones we need to prove are fair, since taking the union of two fair clusters is fair. Thus, to show this is true for all $\ell$ and $v$, we take a union bound over all $\lambda$ values of $\ell$ and $y$ values of $v$. We then find that with probability at least $1 - \frac{2\lambda n/\log n}{cn} = 1 - \frac{2}{\log n}$, all height 1 clusters must be fair, meaning the entire hierarchy must be fair by the union-bound property.

\end{proof}

\begin{proof}[Proof of Theorem~\ref{thm:stoch}]
Let $T^*$ be the optimal tree, let $T_1$ be our $\epsilon$-relatively balanced $\frac{9\gamma}{2\epsilon}$ approximation guaranteed by Theorem~\ref{thm:balance2}, and let $T'$ be our output. By Lemma~\ref{lem:stochchernoff}, $T'$ satisfies our fairness constraints. By Lemmas~\ref{lem:stochdist} and the fact that we only apply one operator, every edge is separated by at most one level abstraction operator of operation cost at most $(1/2+\epsilon)^tn$, but we know from Lemma~\ref{lem:balancedratio} that this is bounded above by $e^{4/(c(1-o(1))}\cdot \frac{3(1-\delta)\ln(cn)}{a\delta^2}$. Because of this, $\cost_{T'}(e) \leq e^{4/(c(1-o(1))}\cdot \frac{3(1-\delta)\ln(cn)}{a\delta^2}\cost_{T_1}(e)$. Summing over all edges yields $\cost(T) \leq e^{4/(c(1-o(1))}\cdot \frac{3(1-\delta)\ln(cn)}{a\delta^2}\cost(T_1) \leq e^{4/(c(1-o(1))}\cdot \frac{3(1-\delta)\ln(cn)}{a\delta^2}\cdot \frac{9\gamma}{2\epsilon}\cost(T^*)$.
\end{proof}

\subsection{$\fairhc$}

This section contains the proofs and additional theoretical discussion regarding $\fairhc$.

\begin{lemma}\label{lem:preprocbalance}
$\treerm$ with $t = \log_{1/2-\epsilon}(1/(2n\epsilon))$ outputs a hierarchy where the $\epsilon$-relatively balanced guarantee holds for all splits except those forming the leaves. Additionally, it admits a proportional cost increase of at most $\frac12ce^{4/(c(1-o(1))}\log_2n$.
\end{lemma}

\begin{proof}[Proof of Lemma~\ref{lem:preprocbalance}]
Say $T$ is our input (i.e., it is $\epsilon$-relatively balanced). Notice that $\treerm$ only modifies $T$'s structure below depth $\log_{1/2-\epsilon}(1/(2n\epsilon))$, which means any balance guarantees hold up to that level. By Lemma~\ref{lem:balancedratio}, for any vertex $v$ at depth $\log_{1/2-\epsilon}(1/(2n\epsilon))$ or above, $n_T(v) \geq (1/2-\epsilon)^{\log_{1/2-\epsilon}(1/(2n\epsilon))}n = 1/(2\epsilon)$. By the definition of $\epsilon$-relative balance, this means that the balance guarantee holds for the split at this vertex. Since all internal vertices in the resulting tree $T'$ are at or above this level, all internal vertices except those with leaf children exhibit the relative balance guarantee.

In order to bound the proportional increase in cost, we must bound the operation cost of the level abstraction. The minimum depth in the abstraction is $\log_{1/2-\epsilon}(1/(2n\epsilon))$. By Lemma~\ref{lem:balancedratio}, this means the maximum cluster size is at most $e^{4/(c(1-o(1)))}/(2\epsilon) = \frac12ce^{4/(c(1-o(1))}\log_2n$. Since the smallest cluster size involved is at least 1, we can then bound the operation cost by this max cluster size, giving our result.
\end{proof}

\begin{lemma}\label{lem:foldbalance}
If $T$ is $\epsilon$-relatively balanced besides the final layer of splits, then the subtrees rooted at all of the root's children in  $\fairhc$ after tree folding are as well.
\end{lemma}

\begin{proof}[Proof of Lemma~\ref{lem:foldbalance}]
Tree folding only involves overlaying the topology of isomorphic trees (ignoring their leaves). Consider a non-root vertex $v$ in $\fairhc$ after tree folding that is also not a parent of leaves. It is the result of merging $k$ vertices $v_1,\ldots, v_k$, and its left and right children $l$ and $r$ are the result of merging $l_1,\ldots,l_k$ and $r_1,\ldots,r_k$ respectively. Due to this:

\begin{align*}
n_{T'}(v) =& \sum_{i\in[k]} n_T(v_i)
\\n_{T'}(l) =& \sum_{i\in[k]} n_T(l_i)
\\n_{T'}(r) =& \sum_{i\in[k]} n_T(r_i)
\end{align*}

We also have, by relative balance, for any $i\in[k]$:

\[(1/2-\epsilon)n_T(v_i) \leq n_T(l_i), n_T(r_i) \leq (1/2+\epsilon)n_T(v_i) \]

A simple combination of these shows that:

\[(1/2-\epsilon)n_T(v) \leq n_T(l), n_T(r) \leq (1/2+\epsilon)n_T(v) \]

This means the split from $v$ to $l$ and $r$ is relatively balanced. We can apply this to all such splits to find the entire new subtree is relatively balanced.
\end{proof}

\begin{proof}[Proof of Lemma~\ref{lem:folddist}]
By Lemma~\ref{lem:balancedratio}, for any vertex $v$ at depth $i \geq \log_2 h$, $n_T(v) \geq (1/2-\epsilon)^{\log_2 h}n$. This can be further simplified using that $\epsilon = 1/(c\log n)$ and $h \leq n$.

\[n_T(v) \geq \frac1{2^{\log_2 h}} \left(1-\frac{2}{c\log_2n}\right)^{\log_2h}n \geq e^{-2/c}n/h \]

Obviously, the largest $n_T(u)$ for any $u$ within our depth bounds is $n$. Thus the level abstraction operation cost is at most $n/(e^{-2/c}n/h) = e^{2/c}h$.
\end{proof}

\begin{proof}[Proof of Lemma~\ref{lem:foldratio}]
That it acts on $k$ trees is obvious. To prove the operation cost, consider $u,v$ in trees $T_i$ and $T_j$ respectively where $\phi_i(u) = \phi_j(v)$. Since we are using the tree isomorphism between the trees, this means that $u$ and $v$ have the same height in $T_i$ and $T_j$ respectively, which also means that they had the same height in the original tree $T$, as the roots of $T_i$ and $T_j$ are both at height $\log_2(h)$ in $T$. Since $u$ and $v$ are on the same level $i \leq \log_{1/2-\epsilon}(1/(2n\epsilon))$, Lemma~\ref{lem:balancedratio} tells us:

\[\frac{n_T(u)}{n_T(v)} \leq  e^{4/(c(1-o(1))}\]

Since this holds for all such pairs $u$ and $v$, this also bounds the tree folding operation cost. Note that when we do this operator, the $\epsilon$-relative balance is held by Lemma~\ref{lem:foldbalance}. Thus, this argument holds across all tree folds in the for loop.
\end{proof}

\begin{proof}[Proof of Lemma~\ref{lem:foldnum}]
If an edge $e = (u,v)$ is separated by the level abstraction, that means $u\land v$ is above depth $\log_2h$. Notice that we recurse on clusters at depth $\log_2 h$, which means on any recursive instance here forward, $e$ will not be contained within the trees, so $e$ cannot be separated. Therefore, $e$ can only be separated by one level abstraction.

Otherwise, notice that the depth of the last internal node is $\log_{1/2-\epsilon}(1/(2n\epsilon))$ by assumption. At each recursive step, we reduce the depth by $\log_2 h$ since we start at subtrees of depth $h$ in the previous tree. Therefore, there are at most $\log_{1/2-\epsilon}(1/(2n\epsilon))/ \log_2 h$ levels of recursion. To simplify this, we use similar methods to Lemma~\ref{lem:foldratio}. which allows us to bound the recursive depth by $\log_2(n)/\log_2(h)$.

At each level of recursion, since $e$ is contained in only one tree, it is only separated by $\lambda$ tree folds. This means $e$ may only be separated by $\lambda\log_2(n)/\log_2(h)$ tree folds.
\end{proof}

\begin{lemma}\label{lem:foldapx}
For an $\epsilon$-relatively balanced hierarchy $T$, Algorithm~\ref{alg:fhc} outputs a tree $T'$ such that:
\[\cost(T') \leq e^{\frac{4\lambda\log_2(n)}{c(1-o(1))\log_2h}+\frac 2c} \cdot h\cost(T)\]
\end{lemma}

\begin{proof}[Proof of Lemma~\ref{lem:foldapx}]
Lemmas~\ref{lem:folddist},~\ref{lem:foldratio}, and~\ref{lem:foldnum} tell us an edge $e$ must only be involved in at most 1 level abstraction of operation cost at most $e^{2/c}h$ and $\lambda\log_2(n)/\log_2(h)$ tree folds of operation cost at most $e^{4/(c(1-o(1))}$ on $k$ trees. By Lemmas~\ref{lem:opabstract} and~\ref{lem:opfold}, this will incur a total proportional cost increase of:

\[\frac{\cost_{T'}(e)}{\cost_T(e)} \leq (e^{4/(c(1-o(1))})^{\lambda\log_2(n)/\log_2(h)} \cdot e^{2/c}h\]

Which, summed over all edges, is equivalent to the desired result.

%Note that $(1/2-\epsilon)^{-\log_2h} \leq (1/2)^{-\log_2h}(1-2/(c\log_2 n))^{-\log_2n} \leq e^{2/c}h$ since $\epsilon = 1/(c\log n)$. We can also insert $\epsilon$ into the fraction to further simplify:

%\[\frac{\cost_{T'}(e)}{\cost_T(e)} \leq (e^{4/c})^{\ln(2n\epsilon)/\log_2h}\cdot e^{2/c}h\]
\end{proof}

\begin{lemma}\label{lem:foldfair}
For an $\epsilon$-relatively balanced hierarchy $T$ over $\ell(V) = c_\ell n=O(n)$ vertices of each color $\ell\in [\lambda]$, $\fairhc$ before recursion ensures that the clustering induced on each depth-1 internal node $v$ of the output tree $T'$ each have $ \frac{c_\ell}{e^{6/c}}\leq \frac{\ell(v)}{\leaves(v)} \leq c_\ell\cdot ( e^{4/c}/(kc_\ell) + e^{6/c})$ for each $\ell\in[\lambda]$.
\end{lemma}

\begin{proof}[Proof of Lemma~\ref{lem:foldfair}]
We start by looking at one tree fold operator. Assume the color we are trying to sort is red. Consider the ordering of the vertices $ \{v_i\}_{i\in[h]}$ from Algorithm~\ref{alg:fhc}, and let $r_i$ the number of red points from $\leaves(v_i)$ and $R$ be the total number of red vertices.

Fix some $i$ and let $v_i'$ be the root vertex of the resulting subtree in the $i$th fold (i.e., the one all the subtrees are mapped onto). We know the vertices involved in this were $v_{i+(j-1)k}$ for all $j\in[h/k]$. Because of the ordering, we know that:

\begin{align*}
r_{i+(j-1)k}/n_T(v_{i+(j-1)k}) \leq& r_{\itvar + (j-2)k}/n_T(v_{\itvar+(j-2)k}),\tag{1}
\\r_{i+(j-1)k}/n_T(v_{i+(j-1)k}) \geq& r_{\itvar + jk}/n_T(v_{\itvar+(j-2)k})\tag{2}
\end{align*}

for all $\itvar\in[h/k]$ assuming $j > 1$ for (1) and $j < k$ for (2). Since these three vertices are at the same height, say $h'$ (with respect to $T$ after rebalancing and before the algorithm began), Lemma~\ref{lem:balancedratio} gives us that:

\begin{align*}
n_T(v_{i+(j-1)k})/n_T(v_{\itvar+(j-2)k}) \leq& \frac{(1+2\epsilon)^{\log_2n}}{(1-2\epsilon)^{\log_2n}},
\\n_T(v_{i+(j-1)k})/n_T(v_{\itvar+jk}) \geq& \frac{(1-2\epsilon)^{\log_2n}}{(1+2\epsilon)^{\log_2n}}
\end{align*}

Combining these with the previous inequalities yield:

\begin{align*}
r_{i+(j-1)k} \leq& \frac{(1+2\epsilon)^{\log_2n}}{(1-2\epsilon)^{\log_2n}}r_{\itvar+(j-2)k},
\\r_{i+(j-1)k} \geq& \frac{(1-2\epsilon)^{\log_2n}}{(1+2\epsilon)^{\log_2n}}r_{\itvar+jk}
\end{align*}

for all $\itvar\in[h/k]$. Since $\epsilon = 1/(c\log_2n)$, this bound can be further simplified to:

\begin{align*}
e^{-4/c}r_{\itvar+jk} \leq r_{i+(j-1)k} \leq& e^{4/c}r_{\itvar+(j-2)k}
\end{align*}

Since these hold for all $y$, we can say that:

\begin{align*}
\frac{k}he^{-4/c}\sum_{\itvar\in[h/k]}r_{\itvar+jk}\leq r_{i+(j-1)k} \leq& \frac{k}he^{4/c}\sum_{\itvar\in[h/k]}r_{\itvar+(j-2)k}
\end{align*}

Another way to think of this is partitioning the vertices (in order) into contiguous chunks of size $h/k$. Then $v_{i+(j-1)k}$ is the $i$th vertex in the $j$th chunk, and we know it has a lower of fraction of red points than clusters in the previous ($(j-2)$th) chunk and a higher fraction than clusters in the next ($j$th) chunk.

Now let $R_{j-1}$ be the number of reds in the entire $j$th chunk (i.e., $R_{j-1} = \sum_{\itvar\in[h/k]} r_{\itvar + (j-1)k}$) Additionally, we can make a comparison between the reds in all chunks and $R$, namely, $\sum_{j\in [k]} R_{j-1} = R$.

Putting our two previous results together, for our fixed $i$:

\begin{align*}
\sum_{j\in[k]} r_{i+(j-1)k} \leq& r_{i} + \frac{k}{h}e^{4/c}\sum_{j\in[k]} R_{j-1}  \\=& r_{i} + \frac{k}{h}e^{4/c}R,
\\\sum_{j\in[k]} r_{i+(j-1)k} \geq& r_{h - h/k + i} + \frac{k}{h}e^{-4/c}\sum_{j\in[k]} R_{j-1}  \\=& \frac{k}{h}e^{-4/c}R
\end{align*}

Notice that if everything were perfectly balanced, $\frac{k}{h}R$ is exactly the number of reds we would want in $\leaves(v_i')$. We now must bound $r_{i}$. Unfortunately, it could be an entirely red cluster, so this is only bounded by the size of the cluster at depth $\log_2h$, which we get from Lemma~\ref{lem:balancedratio}.

%\[r_{i} \leq |C_{i}| \leq e^{2/c}(1+o(1))(n/h)\]

\begin{align*}
r_{i} \leq& n_T(v_i) 
\\\leq& (1/2+\epsilon)^{\log_2h}n 
\\=& 2^{-\log_2h}(1+2\epsilon)^{\log_2h}n 
\\\leq& e^{2/c}n/h
\end{align*}

Note the final inequality comes from the fact that $h \leq n$ and $\epsilon = 1/(c\log n)$. Now note that we are given $R = c_Rn$ for some $c_R = O(1)$. We can sub this in.

\[r_{i} \leq e^{2/c}R/(c_Rh)\]

Now, notice we are actually looking for the fraction of red points in the cluster. Since Lemma~\ref{lem:balancedratio} gives us that $n_{T'}(v_i') \geq k(1-2\epsilon)^{\log_2h}n/h \geq ke^{-2/c}n/h$ and $n_{T'}(v_i') \leq k(1+2\epsilon)^{\log_2h}n/h \leq ke^{2/c}n/h$ (applying the same logic as the upper bound to $n_T(v_i)$, $k$ times), we get:

\begin{align*}
\frac{\sum_{j\in[k]} r_{i+jk}}{n_{T'}(v_i')} \leq& \frac{ e^{2/c}R/(c_Rh) + \frac{k}{h}e^{4/c}R}{ke^{-2/c}(n/h)}
\\=& \frac Rn\cdot \left(\frac{ e^{4/c}/c_R}{k} + e^{6/c}
\right),
\\\frac{\sum_{j\in[k]} r_{i+jk}}{n_{T'}(v_i')} \geq& \frac{\frac{k}{h}e^{-4/c}R}{ke^{2/c}(n/h)}
\\=& \frac Rn\cdot \frac{1}{e^{6/c}}
\end{align*}

This completes the proof for one tree fold under the observation that $\frac Rn = c_\ell$ if red is $\ell$. The same (if not stronger) bounds hold for all subsequent $\lambda$ tree folds for each color. Note that as we proceed, this bound will not be disrupted since merging two clusters that guarantees the same upper bound on the fraction of red points still guarantees the same bound.
\end{proof}

\begin{proof}[Proof of Lemma~\ref{lem:fair}]
Clearly, the most imbalanced clusters in this process will be the clusters in the final level of the hierarchy. By Lemma~\ref{lem:foldfair}, when we recurse, we have at most an $ \frac{c_\ell}{e^{6/c}}\leq \frac{\ell(v)}{\leaves(v)} \leq c_\ell\cdot ( e^{4/c}/(kc_\ell) + e^{6/c})$
 fraction of vertices of color $\ell$ for each $\ell\in[\lambda]$. Clearly, after at most $\log_2n/ \log_2 h = \log_hn$ recursive levels guaranteed by Lemma~\ref{lem:foldnum}, our bound becomes:

 \[\frac{c_\ell}{e^{6t\log_hn/c}}\leq \frac{\ell(v)}{\leaves(v)} \leq c_\ell\cdot ( e^{4/c}/(kc_\ell) + e^{6/c})^{\log_hn}.\]
\end{proof}

\begin{proof}[Proof of Theorem~\ref{thm:main}]
Let $T^*$ be the optimal tree, let $T_1$ be our input tree which is a $ce^{4/(c(1-o(1))}\log_2n\cdot \frac{9\gamma}{4\epsilon}$ approximation guaranteed by Theorem~\ref{thm:stoch} but using $t = \log_{1/2-\epsilon}(1/(2n\epsilon))$ (this was shown more explicitly in Lemma~\ref{lem:preprocbalance}), and let $T'$ be our output. By Lemma~\ref{lem:fair}, $T'$ satisfies our fairness constraints. By Lemmas~\ref{lem:folddist},~\ref{lem:foldratio}, and~\ref{lem:foldnum}, every edge is separated by at most 1 level abstraction of max operation cost $e^{2/c}n/h$ and $\log_2(n)/\log_2(h)$ tree folds of operation cost at most $e^{4/(c(1-o(1))}$ on $k$ subtrees. Lemma~\ref{lem:foldapx} immediately tells us:

\[\cost(T') \leq e^{\frac{4\lambda\log_2n}{c(1-o(1))\log_2h}+\frac 2c} \cdot h\cost(T_1)\]

 Combining this with the approximation guaranteed by $T_1$:

\[\cost(T') \leq e^{\frac{4\lambda\log_2n}{c(1-o(1))\log_2h}+\frac 2c} \cdot hce^{4/(c(1-o(1))}\log_2n\cdot \frac{9\gamma}{4\epsilon}\cost(T^*)\]

Simplifying and plugging in $h=n^\delta,\epsilon=1/(c\log_2n)$ yields the desired result.
\end{proof}
\section{Runtime}

Here we analyze the runtime of our four algorithms. Recall that before each of these algorithms, we run a black-box cost-approximate hierarchical clustering algorithm as well as all previous algorithms. For simplicity, here we will present the contribution of each algorithm to the runtime.

\paragraph{Theorem~\ref{thm:balance}:} This algorithm starts at the root, traverses down one side of the tree until a certain sized cluster is found, and then applies a tree rebalance. It then recurses on each child. The length of traversal is bounded by $O(n)$, and a single tree rebalance operation requires some simple constant-time pointer operations. As this is run from each vertex in the tree, the total runtime is $O(n^2)$.

\paragraph{Theorem~\ref{thm:balance2}:} As in the previous algorithm, here we do a computation at each of the $O(n)$ nodes in the tree. At each node, we apply subtree search until the desired balance is achieved. If $\delta$ is the current balance, we reduce this to at most $2\delta/3$ at each step. Thus, this will require a total of $O(\log(1/\epsilon))$ steps to complete. Each subtree search operation requires two, $O(n)$-length traversals to find the place to insert and delete. Otherwise, it is constant-time pointer math. Thus, the algorithm runs in $O(n^2\log(1/\epsilon))$ time.

\paragraph{Theorem~\ref{thm:stoch}:} This algorithm is quite, simple, as we are simply deleting some set of low nodes in the tree. Thus it only requires $O(n)$ time.

\paragraph{Theorem~\ref{thm:main}:} Again, we execute a computation for at most $O(n)$ nodes in the tree. By a similar logic as before, tree abstraction steps require $O(n)$ time. It is not too hard to see that computing the fraction of red and blue vertices in each considered cluster and then sorting them accordingly also requires $O(n)$ time. Finally, we fold the vertices on top of each other. The isomorphism used for folding can be found by simply indexing the vertices in each subtree, and then applied quite directly, which also takes $O(n)$ time. Thus this requires only $O(n^2)$ time.

Therefore, the entire final algorithm (without the blackbox step) is bounded by the computation time from Theorem~\ref{thm:balance2}, which is $O(n^2\log(1/\epsilon))$. Intuitively, $\epsilon$ is bounded by $\epsilon > 1/n$, thus this becomes $O(n^2\log n)$.
\section{Additional Experiments}
We here provide the figures and results omitted from the main text due to space constraints.

\subsection{Bank Data}
We begin with the supplementary figure that complements Figure~\ref{fig:fairness} in the final section. This figure depicts the fairness of each clusterin the hierarchy constructed by Algorithm~\ref{alg:fhc} on the \emph{Bank} data. We see an equivalent concentration about the true balance ratio of 1:3.

\begin{figure}[h]%[!htb]
	\center{\includegraphics[width=0.95\linewidth]{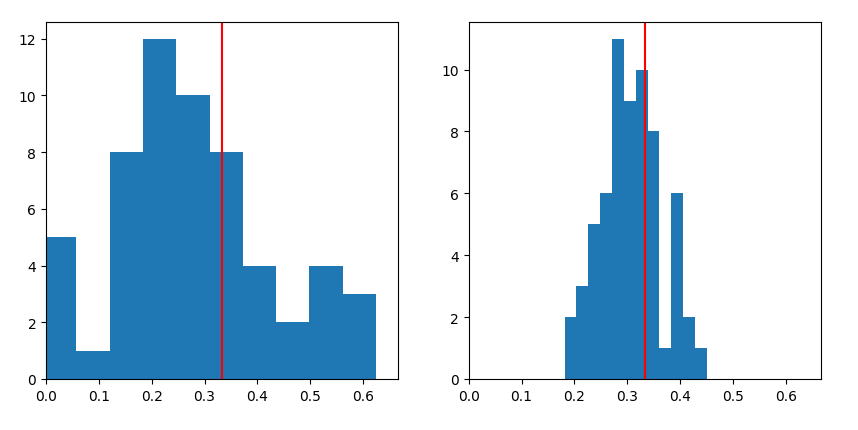}}
    \caption{\label{fig:fairness_supp} Histogram of cluster balances after tree manipulation by Algorithm~\ref{alg:fhc}. The left plot depicts the balances after applying the average-  linkage algorithm and the right shows the result of applying our algorithm. The vertical red line indicates the balance of the dataset. Parameters were set to $c = 4, \delta = \frac38, k = 4$ for the above clustering result.
    }

\vspace{-3mm}
\end{figure}

\subsection{Fairness Results for Parameter Sweep}
We additionally provide the fairness results presented in Figure~\ref{fig:fairness} for the other parameter sets plotted in Figure~\ref{fig:n128_params} for completeness.

\begin{figure}[h]%[!htb]
	\center{\includegraphics[width=0.36\linewidth]{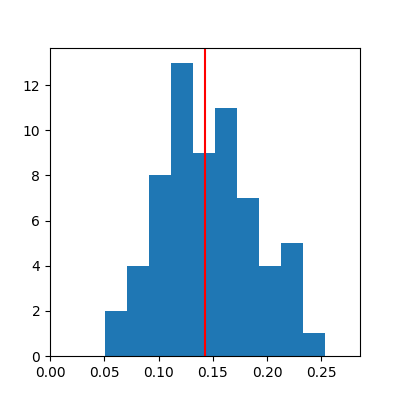}}
    \caption{\label{fig:fairness_supp} Parameters set to $c = 1, k = 8, \delta = \frac38$.
    }

\vspace{-3mm}
\end{figure}

\begin{figure}[h]%[!htb]
	\center{\includegraphics[width=0.36\linewidth]{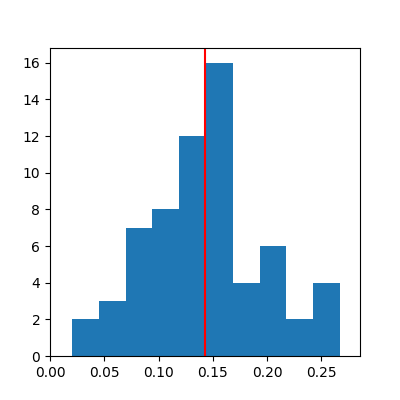}}
    \caption{\label{fig:fairness_supp} Parameters set to $c = 2, k = 8, \delta = \frac38$.
    }

\vspace{-3mm}
\end{figure}

\begin{figure}[h]%[!htb]
	\center{\includegraphics[width=0.36\linewidth]{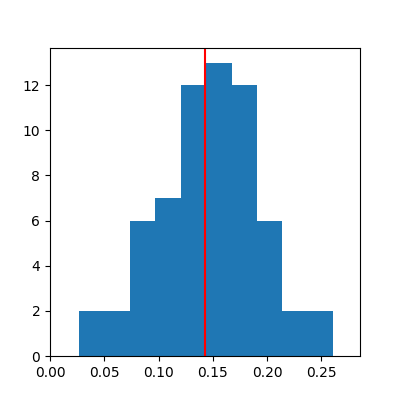}}
    \caption{\label{fig:fairness_supp} Parameters set to $c = 4, k = 8, \delta = \frac38$.
    }

\vspace{-3mm}
\end{figure}

\begin{figure}[h]%[!htb]
	\center{\includegraphics[width=0.36\linewidth]{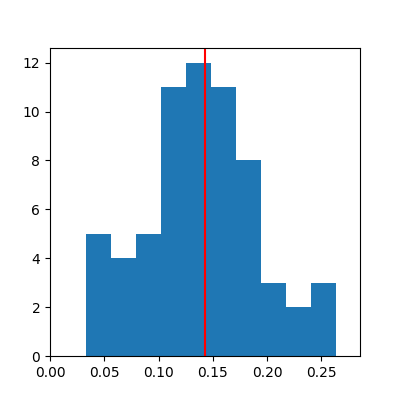}}
    \caption{\label{fig:fairness_supp} Parameters set to $c = 8, k = 8, \delta = \frac38$.
    }

\vspace{-3mm}
\end{figure}

\begin{figure}[h]%[!htb]
	\center{\includegraphics[width=0.36\linewidth]{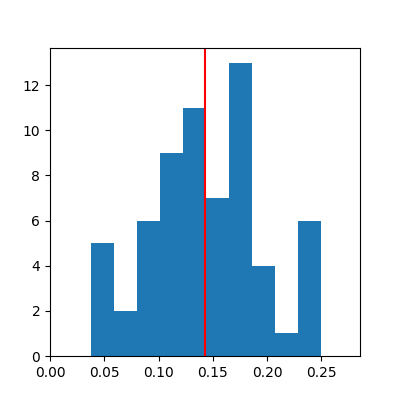}}
    \caption{\label{fig:fairness_supp} Parameters set to $c = 16, k = 8, \delta = \frac38$.
    }

\vspace{-3mm}
\end{figure}

\begin{figure}[h]%[!htb]
	\center{\includegraphics[width=0.36\linewidth]{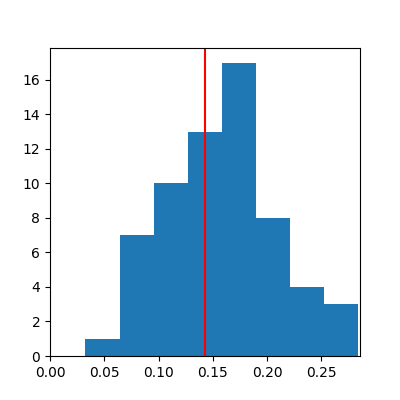}}
    \caption{\label{fig:fairness_supp} Parameters set to $c = 4, k = 4, \delta = \frac78$.
    }

\vspace{-3mm}
\end{figure}

\begin{figure}[h]%[!htb]
	\center{\includegraphics[width=0.36\linewidth]{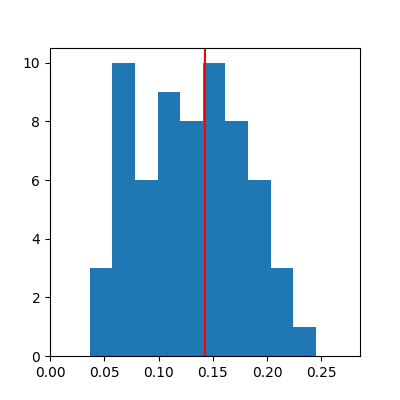}}
    \caption{\label{fig:fairness_supp} Parameters set to $c = 4, k = 8, \delta = \frac78$.
    }

\vspace{-3mm}
\end{figure}

\begin{figure}[h]%[!htb]
	\center{\includegraphics[width=0.36\linewidth]{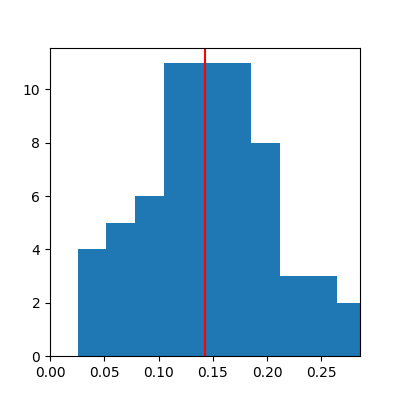}}
    \caption{\label{fig:fairness_supp} Parameters set to $c = 4, k = 16, \delta = \frac78$.
    }

\vspace{-3mm}
\end{figure}

\begin{figure}[h]%[!htb]
	\center{\includegraphics[width=0.36\linewidth]{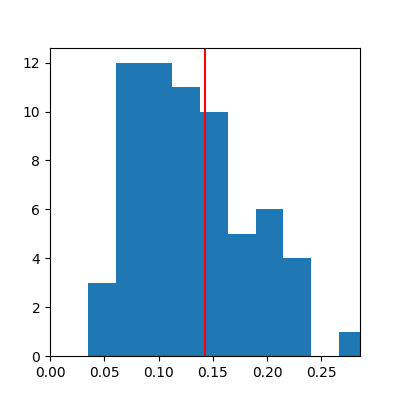}}
    \caption{\label{fig:fairness_supp} Parameters set to $c = 4, k = 4, \delta = \frac38$.
    }

\vspace{-3mm}
\end{figure}

\begin{figure}[h]%[!htb]
	\center{\includegraphics[width=0.36\linewidth]{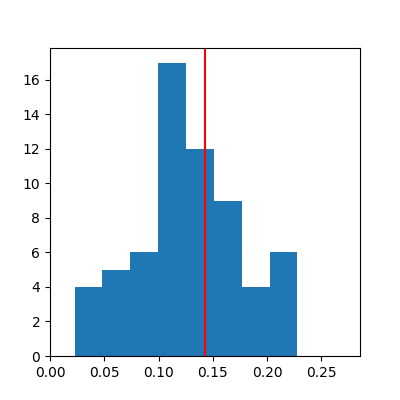}}
    \caption{\label{fig:fairness_supp} Parameters set to $c = 4, k = 4, \delta = \frac48$.
    }

\vspace{-3mm}
\end{figure}

\begin{figure}[h]%[!htb]
	\center{\includegraphics[width=0.36\linewidth]{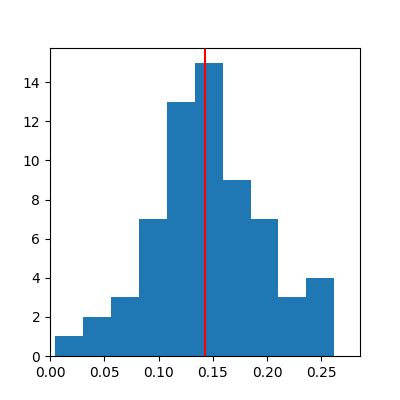}}
    \caption{\label{fig:fairness_supp} Parameters set to $c = 4, k = 4, \delta = \frac58$.
    }

\vspace{-3mm}
\end{figure}

\begin{figure}[h]%[!htb]
	\center{\includegraphics[width=0.36\linewidth]{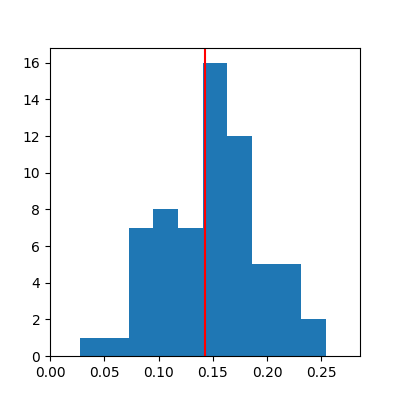}}
    \caption{\label{fig:fairness_supp} Parameters set to $c = 4, k = 4, \delta = \frac68$.
    }

\vspace{-3mm}
\end{figure}

\begin{figure}[h]%[!htb]
	\center{\includegraphics[width=0.36\linewidth]{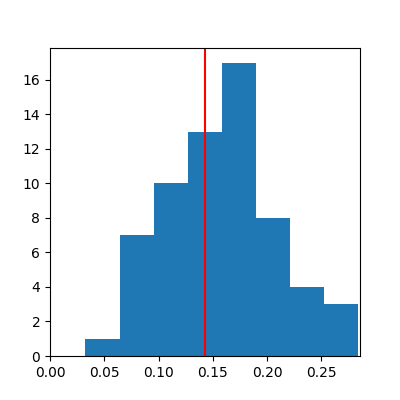}}
    \caption{\label{fig:fairness_supp} Parameters set to $c = 4, k = 4, \delta = \frac78$.
    }

\vspace{-3mm}
\end{figure}

% Adding experiments for potential ICML rebuttal:
\clearpage
\newpage
\onecolumn
\section{Further Experimentation (Rebuttal Draft)}

We here present further experimental results on a higher-dimensional data sample as compared to the main text. We first present the histogram of cluster balances after application of the average linkage algorithm for $n = 1024$ samples.

\begin{figure}[h]
    \centering
    \includegraphics[width=0.33\linewidth]{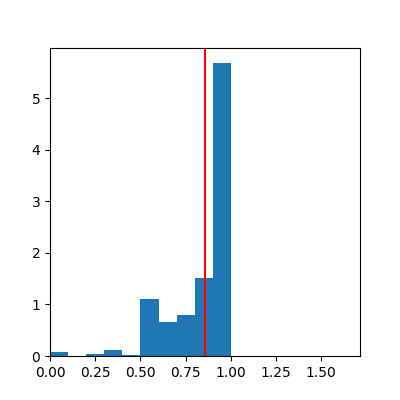}
    \caption{Average linkage balances for $n = 1024$ samples}
    \label{fig:avg_lkg_1024}
\end{figure}

We now proceed to apply our fair hierarchical clustering algorithm (Algorithm \ref{alg:fhc}) on the constructed base cluster tree from average linkage resulted from the above. The algorithm was run with a wide variety of parameter tuples $(c, h, k)$ where, as noted in the text, $h = n^\delta$. To further reiterate: the parameter $c$ is used to define the $\varepsilon$ cluster balance, $h$ the number of clusters and $k$ the number of trees folded in the rebalance procedure.

\begin{figure}[h]
    \center{\includegraphics[width=0.33\linewidth]{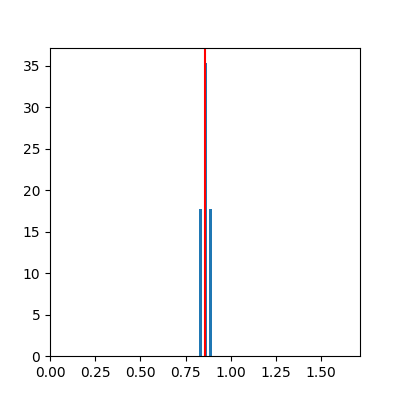}}
    \caption{Result of running Algorithm~\ref{alg:fhc} on $n=1024$ samples with parameter tuple $(c,h,k) = (1,4,2)$.}
    \label{fig:fairhc_142}
\end{figure}

\begin{figure}[h]
    \centering
    \includegraphics[width=0.33\linewidth]{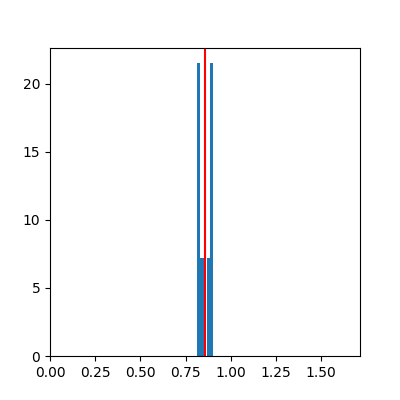}
    \caption{Result of running Algorithm~\ref{alg:fhc} on $n=1024$ samples with parameter tuple $(c,h,k) = (1,8,2)$.}
    \label{fig:fairhc_182}
\end{figure}

\begin{figure}[h]
    \centering
    \includegraphics[width=0.33\linewidth]{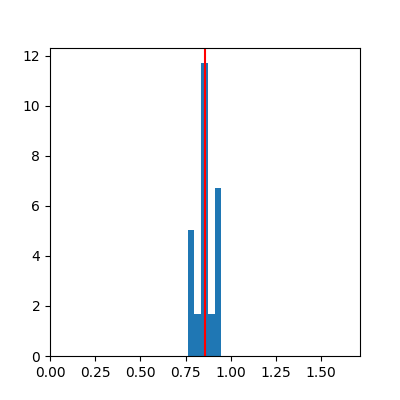}
    \caption{Result of running Algorithm~\ref{alg:fhc} on $n=1024$ samples with parameter tuple $(c,h,k) = (1,16,2)$.}
    \label{fig:fairhc_1162}
\end{figure}

\begin{figure}[h]
    \centering
    \includegraphics[width=0.33\linewidth]{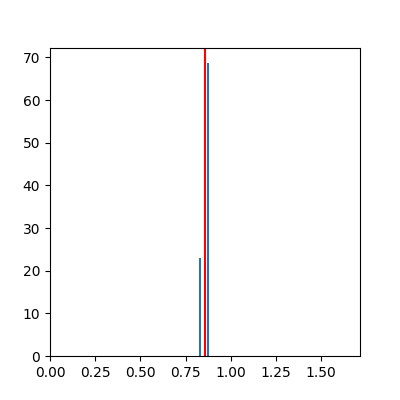}
    \caption{Result of running Algorithm~\ref{alg:fhc} on $n=1024$ samples with parameter tuple $(c,h,k) = (2,4,2)$.}
    \label{fig:fairhc_242}
\end{figure}

\begin{figure}[h]
    \centering
    \includegraphics[width=0.33\linewidth]{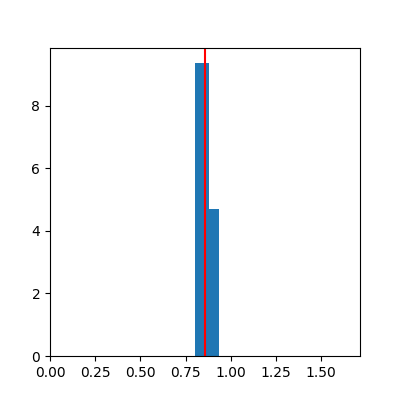}
    \caption{Result of running Algorithm~\ref{alg:fhc} on $n=1024$ samples with parameter tuple $(c,h,k) = (2,8,2)$.}
    \label{fig:fairhc_282}
\end{figure}

\begin{figure}[h]
    \centering
    \includegraphics[width=0.33\linewidth]{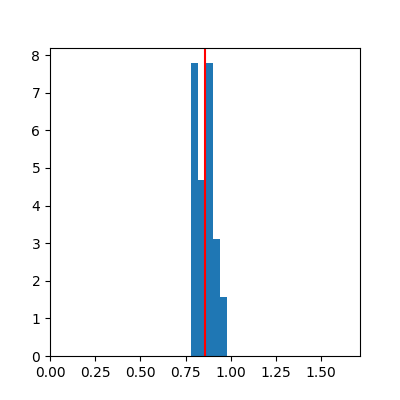}
    \caption{Result of running Algorithm~\ref{alg:fhc} on $n=1024$ samples with parameter tuple $(c,h,k) = (2,16,2)$.}
    \label{fig:fairhc_2162}
\end{figure}

\begin{figure}[h]
    \centering
    \includegraphics[width=0.33\linewidth]{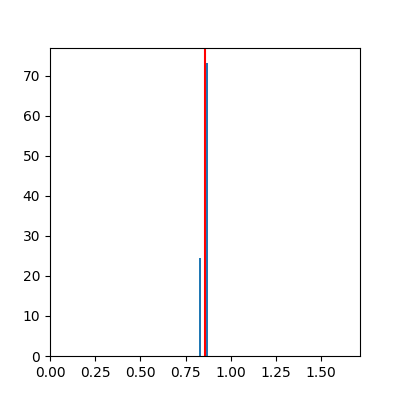}
    \caption{Result of running Algorithm~\ref{alg:fhc} on $n=1024$ samples with parameter tuple $(c,h,k) = (4,4,2)$.}
    \label{fig:fairhc_442}
\end{figure}

\begin{figure}[h]
    \centering
    \includegraphics[width=0.33\linewidth]{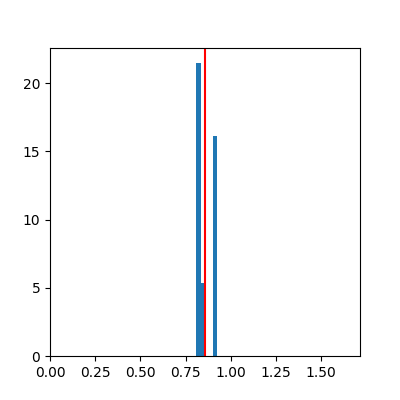}
    \caption{Result of running Algorithm~\ref{alg:fhc} on $n=1024$ samples with parameter tuple $(c,h,k) = (4,8,2)$.}
    \label{fig:fairhc_482}
\end{figure}

\begin{figure}[h]
    \centering
    \includegraphics[width=0.33\linewidth]{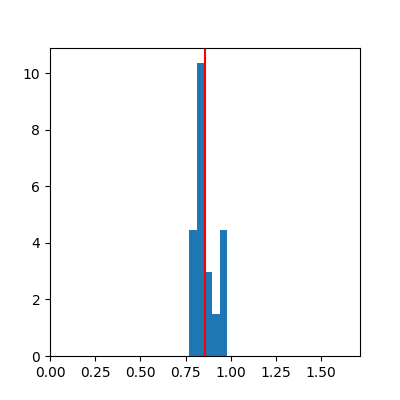}
    \caption{Result of running Algorithm~\ref{alg:fhc} on $n=1024$ samples with parameter tuple $(c,h,k) = (4,16,2)$.}
    \label{fig:fairhc_4162}
\end{figure}

\begin{figure}[h]
    \centering
    \includegraphics[width=0.33\linewidth]{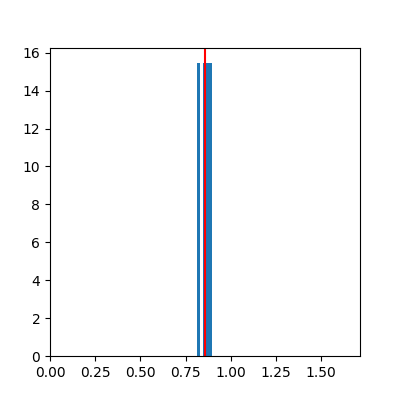}
    \caption{Result of running Algorithm~\ref{alg:fhc} on $n=1024$ samples with parameter tuple $(c,h,k) = (8,4,2)$.}
    \label{fig:fairhc_842}
\end{figure}

\begin{figure}[h]
    \centering
    \includegraphics[width=0.33\linewidth]{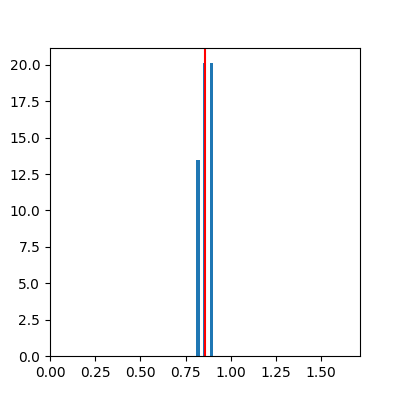}
    \caption{Result of running Algorithm~\ref{alg:fhc} on $n=1024$ samples with parameter tuple $(c,h,k) = (8,8,2)$.}
    \label{fig:fairhc_882}
\end{figure}

\begin{figure}[h]
    \centering
    \includegraphics[width=0.33\linewidth]{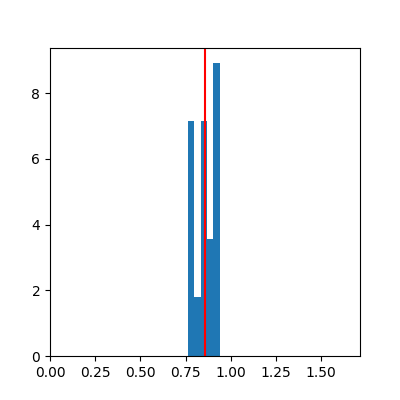}
    \caption{Result of running Algorithm~\ref{alg:fhc} on $n=1024$ samples with parameter tuple $(c,h,k) = (8,16,2)$.}
    \label{fig:fairhc_8162}
\end{figure}

\section{Runtime and Cost Experiments}
\begin{table}[]
\centering
\begin{tabular}{c|c|c|c|c}
\textbf{$c$} & \textbf{$h$} & \textbf{$k$} & \textbf{Runtime ($s$)} & \textbf{Cost} \\ \hline
1            & 4            & 2            & 1.483                  & 129.916       \\
1            & 8            & 2            & 1.351                  & 164.475       \\
1            & 16           & 2            & 1.314                  & 201.268       \\
2            & 4            & 2            & 1.535                  & 146.919       \\
2            & 8            & 2            & 1.351                  & 149.92        \\
2            & 16           & 2            & 1.313                  & 208.099       \\
4            & 4            & 2            & 1.534                  & 150.02        \\
4            & 8            & 2            & 1.353                  & 156.071       \\
4            & 16           & 2            & 1.305                  & 224.063       \\
8            & 4            & 2            & 1.454                  & 136.179       \\
8            & 8            & 2            & 1.302                  & 225.995      
\end{tabular}
\end{table}

\end{document}

% This document was modified from the file originally made available by
% Pat Langley and Andrea Danyluk for ICML-2K. This version was created
% by Iain Murray in 2018, and modified by Alexandre Bouchard in
% 2019 and 2021 and by Csaba Szepesvari, Gang Niu and Sivan Sabato in 2022.
% Modified again in 2023 by Sivan Sabato and Jonathan Scarlett.
% Previous contributors include Dan Roy, Lise Getoor and Tobias
% Scheffer, which was slightly modified from the 2010 version by
% Thorsten Joachims & Johannes Fuernkranz, slightly modified from the
% 2009 version by Kiri Wagstaff and Sam Roweis's 2008 version, which is
% slightly modified from Prasad Tadepalli's 2007 version which is a
% lightly changed version of the previous year's version by Andrew
% Moore, which was in turn edited from those of Kristian Kersting and
% Codrina Lauth. Alex Smola contributed to the algorithmic style files.